\documentclass[letterpaper]{article} 
\usepackage{aaai25}  
\usepackage{times}  
\usepackage{helvet}  
\usepackage{courier}  
\usepackage[hyphens]{url}  
\usepackage{graphicx} 
\usepackage{natbib}  
\usepackage{caption} 
\usepackage{algorithm}
\usepackage[noend]{algorithmic} 
\usepackage{multirow}
\usepackage{graphicx}
\usepackage{subcaption}
\usepackage{pifont}

\usepackage{my}
\usepackage{booktabs}
\usepackage{amsmath}
\usepackage{makecell}
\allowdisplaybreaks  

\usepackage[breakable]{tcolorbox}
\usepackage{enumitem}
\usepackage{tikz}

\usepackage{newfloat}
\usepackage{listings}
\DeclareCaptionStyle{ruled}{labelfont=normalfont,labelsep=colon,strut=off} 
\floatstyle{ruled}
\newfloat{listing}{tb}{lst}{}
\floatname{listing}{Listing}
\makeatletter
\newcounter{corrfn}
\def\corrauthor{%
  \ifnum\value{corrfn}=0%
    \thanks{Corresponding Author}%
    \setcounter{corrfn}{\value{footnote}}%
  \else%
    \footnotemark[\value{corrfn}]%
  \fi%
}
\makeatother

\title{MRBTP: Efficient Multi-Robot Behavior Tree Planning and Collaboration}
\author{
    Yishuai Cai\equalcontrib,
    Xinglin Chen\equalcontrib,
    Zhongxuan Cai\corrauthor,
    Yunxin Mao,\\
    Minglong Li\corrauthor ,
    Wenjing Yang,
    Ji Wang
}

\affiliations{
    \textsuperscript{\rm }College of Computer Science and Technology, National University of Defense Technology\\

    \{caiyishuai, chenxinglin, caizhongxuan, maoyunxin, liminglong10, wenjing.yang, wj\}@nudt.edu.cn
    
}

\usepackage{bibentry}

\begin{document}

\maketitle

\begin{abstract}


Multi-robot task planning and collaboration are critical challenges in robotics. While Behavior Trees (BTs) have been established as a popular control architecture and are plannable for a single robot, the development of effective multi-robot BT planning algorithms remains challenging due to the complexity of coordinating diverse action spaces. We propose the Multi-Robot Behavior Tree Planning (MRBTP) algorithm, with theoretical guarantees of both soundness and completeness. MRBTP features cross-tree expansion to coordinate heterogeneous actions across different BTs to achieve the team's goal. For homogeneous actions, we retain backup structures among BTs to ensure robustness and prevent redundant execution through intention sharing. While MRBTP is capable of generating BTs for both homogeneous and heterogeneous robot teams, its efficiency can be further improved. We then propose an optional plugin for MRBTP when Large Language Models (LLMs) are available to reason goal-related actions for each robot. These relevant actions can be pre-planned to form long-horizon subtrees, significantly enhancing the planning speed and collaboration efficiency of MRBTP. We evaluate our algorithm in warehouse management and everyday service scenarios. Results demonstrate MRBTP's robustness and execution efficiency under varying settings, as well as the ability of the pre-trained LLM to generate effective task-specific subtrees for MRBTP.

\end{abstract}

%
\begin{links}
    \link{Code}{https://github.com/DIDS-EI/MRBTP}
\end{links}

\section{Introduction}
Multi-robot systems (MRS) that involve robots with diverse capabilities offer the potential for improved performance and fault tolerance compared to single-robot solutions \cite{colledanchise2016advantages}. Developing an autonomous MRS requires an efficient and robust control architecture, along with methods to adapt them for specific tasks. Behavior Trees (BTs) have emerged as a popular control architecture due to their modularity, interpretability, reactivity, and robustness, making them well-suited for both single- and multi-robot systems \cite{heppner2024behavior,heppnerl2023distributed,neupane2019learning,colledanchise2016advantages}. As the potential of BTs gains more attention, various methods for automatically generating BTs have been proposed, including evolutionary computing \cite{neupane2019learning,colledanchise2019learning}, reinforcement learning \cite{banerjee2018autonomous,pereira2015framework}, and BT synthesis \cite{tadewos2022specificationguided,neupane2023designing}. Among these methods, BT planning \cite{chen2024integrating,chen2024efficient,cai2021bt,colledanchise2019blended} has advantages in leveraging interpretable action models and producing reliable BTs to achieve goals, which make it a promising approach to generate BTs for autonomous robot systems.

\begin{figure}[t]
	\centering
	\includegraphics[width=0.47\textwidth]{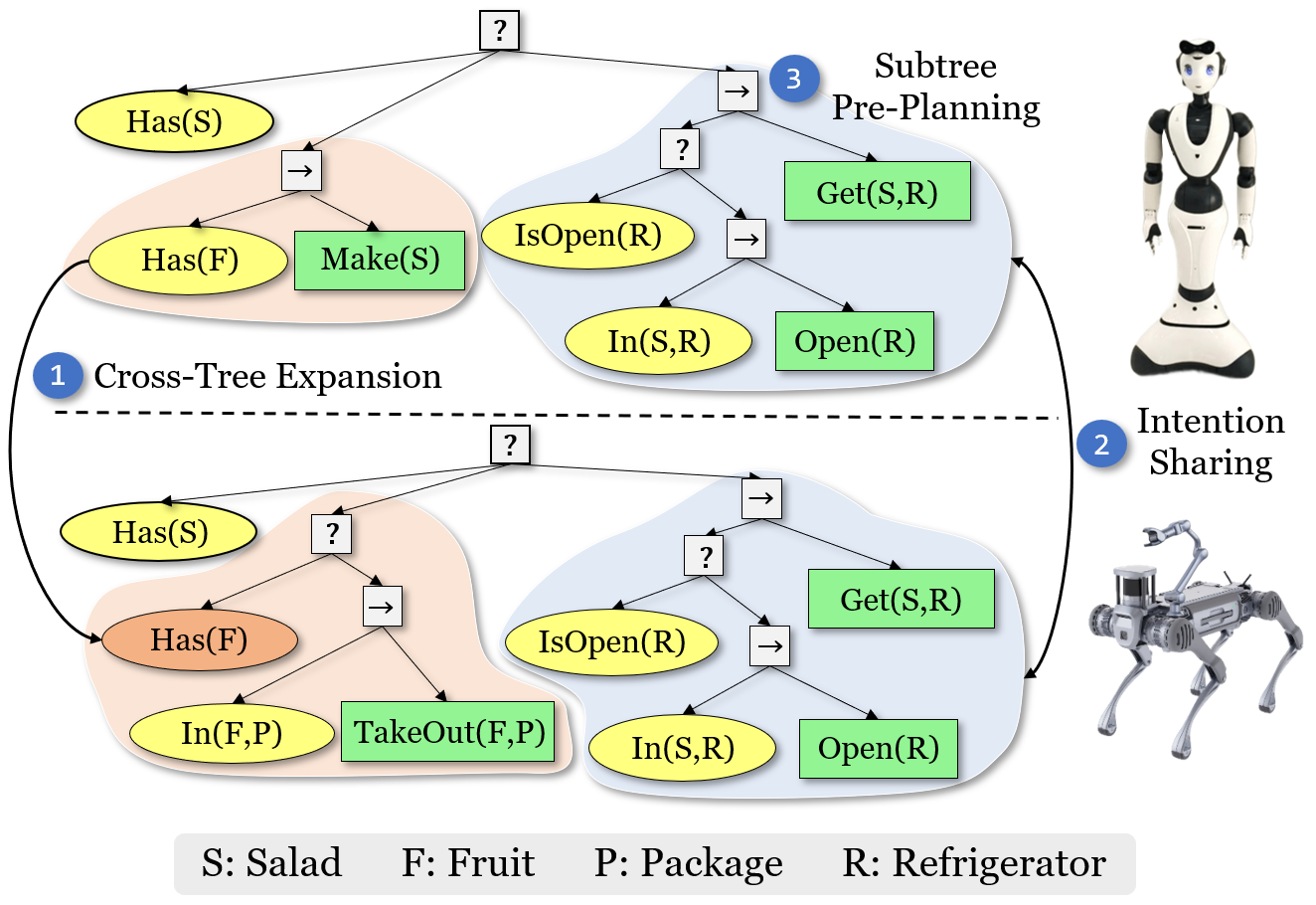}
	\caption{An example of two BTs planned by MRBTP: (1) Cross-tree expansion, (2) Intention sharing, (3) Optional plugin: subtree pre-planning.}
	\label{fig:introduction}
\end{figure}

However, current BT planning focuses primarily on a single robot, and the development of effective multi-robot BT planning algorithms remains challenging. The challenges mainly arise from two aspects:
\begin{itemize}
\item For heterogeneous actions, how to coordinate them across different BTs to accomplish the team's goal.
\item For homogeneous actions, how to use them to improve fault tolerance without redundant execution.
\end{itemize}

In this paper, we propose Multi-Robot Behavior Tree Planning (MRBTP), the first sound and complete algorithm for generating reliable and robust BTs for MRS. MRBTP addresses the above challenges as follows:
\begin{itemize}
\item We employ cross-tree expansion, where the condition expanded in one BT will be further expanded by all BTs. This means one robot may take an action to satisfy another's precondition, enabling multi-tree collaboration.
\item We allow backup structures to be expanded by robots with homogeneous actions to ensure fault tolerance, while using intention sharing to avoid redundant execution. During execution, each robot broadcasts its current action so that others can predict its effects and avoid performing actions with the same effects.
\end{itemize}

As shown in Figure \ref{fig:introduction}, in an everyday service scenario, the team's goal is to prepare a \constant{Salad}. The humanoid robot can perform the action \constant{Make(Fruit,Salad)}, if the precondition \constant{Has(Fruit)} is satisfied. Although it cannot do \constant{Unload(Fruit,Package)}, it can push \constant{Has} \constant{(Fruit)} to the planning queue, and another quadruped robot will expand this action through cross-tree expansion, enabling multi-tree coordination. 

In another case where \constant{In(Salad,Refrigerator)} is satisfied, and both robots can do \constant{Open(Refrigerator)} and \constant{Get(Salad)}. In this case, MRBTP will expand the same structure in both trees to ensure failure tolerance. If both robots are available, the humanoid robot with higher priority will do \constant{Open}\constant{(Refrigerator)} and share its intention. The quadruped robot will then assume \constant{IsOpen}\constant{(Refrigerator)} is true and walk to the \constant{Refrigerator}, waiting to do \constant{Get(Salad,} \constant{Refrigerator)} as long as the \constant{Refrigerator} is truly open. The intention sharing ensures parallelization and improves execution efficiency of the robot team.

Although MRBTP is a domain-independent algorithm, it is possible to enhance the planning and execution efficiency if Large Language Models (LLMs) are available for domain-dependent reasoning. Therefore, we further propose an optional plugin named subtree pre-planning. Assuming the LLM can reason some useful actions for each robot according to its capabilities, we can use these actions to plan useful subtree structures quickly before the long-horizon planning process. These subtrees can not only increase the planning speed, but can also decrease the communication expenses during execution. Experiments in warehouse management and everyday service scenarios demonstrate MRBTP's robustness and execution efficiency under varying settings, as well as the ability of pre-trained LLMs to generate effective task-specific subtrees for MRBTP.


\section{Background}
\paragraph{Behavior Tree.} 
A BT is a directed rooted tree where the execution nodes interact with the environment and the control flow nodes handle the triggering logic of their children\cite{colledanchise2018behavior}. 
At each time step, the BT initiates a tick that goes through control nodes, determining the action that the robot will execute according to the environmental state. This paper mainly focuses on four typical BT nodes:

\begin{itemize}
\item \textbf{Condition} \includegraphics[height=0.75em]{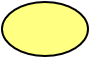}: 
An execution node that checks whether the environment state satisfies the specified condition, returning either {\ttfamily success} or {\ttfamily failure} accordingly.

\item \textbf{Action} \includegraphics[height=0.75em]{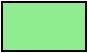}: An execution node that controls the robot to perform an action, returning {\ttfamily success}, {\ttfamily failure}, {\ttfamily running} depending upon the outcome of execution.

\item \textbf{Sequence} \includegraphics[height=0.75em]{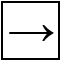}: A control flow node that only returns {\ttfamily success} if all its children succeed. Otherwise, it ticks its children from left to right, and the first child to return {\ttfamily failure} or {\ttfamily running} will determine its return status. 

\item \textbf{Fallback} \includegraphics[height=0.75em]{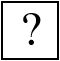}:  A control flow node with logic opposite to the sequence node. It returns {\ttfamily failure} only if all of its children fail. If not, the first occurrence of {\ttfamily success} or {\ttfamily running} during ticking becomes its return status.
\end{itemize}

\paragraph{BT Planning.}
In BT planning for a single robot \cite{cai2021bt}, we represent a BT as a three-tuple $\mathcal{T} = <f, r, \Delta t>$. $f:2^{n}\rightarrow 2^{n}$ is its effect on the environment state, $\Delta t$ is the time step, and $r:2^{n}\mapsto \{ $\constant{S}, \constant{R}, \constant{F}\} partitions states into three regions, where $\mathcal{T}$ returns success, running, failure, respectively.

Then the BT planning problem can be described as: \(<\mathcal{S},\mathcal{L},\mathcal{A},\mathcal{M}, s_0,g>\), where \( \mathcal{S} \) is the finite set of environment states, $\mathcal{L}$ is the finite set of literals that form states, \( \mathcal{A} \) is the finite set of actions, $\mathcal{M}$ is the action model, $s_0$  is the initial state, $g$ is the goal condition.

A condition $c$ in BT is usually a subset of a state $s$. If $c\subseteq s$, it is said condition $c$ holds in that state $s$. The state transition affected by action $a\in \mathcal{A}$ can be defined as a triplet \( \mathcal{M}(a)=<pre(a),add(a),del(a)> \), comprising the precondition, add effects, and delete effects of the action. If $a$ is finished after $k$ time step, the subsequent state $s_{t'}$ will be:
\begin{equation}\label{eqn:s_f}
	s_{t'}=f_a(s_t)=s_t\cup add(a)\setminus del(a), t'=t + k
\end{equation}

\section{Problem Formulation}
We first extend the BT representation from a single robot to a multi-robot system.

\begin{definition}[Multi-BT System]
A $n$-robot BT system is a four-tuple $\left<\Phi, f_\Phi, r_\Phi, \Delta t_\Phi\right>$, where $\Phi = \left\{ \mathcal{T}_i \right\}_{i=1}^n$ is the set of BTs, $f_\Phi: \mathcal{S} \mapsto \mathcal{S}$ is the team state transition function, $\Delta t_\Phi$ is the team time step, $r_\Phi: \mathcal{S} \mapsto \{$ \constant{S}, \constant{R}, \constant{F} $\}$ is the team region partition.
\end{definition}
Due to variability in hardware performance, we allow each robot's BT to have a different response frequency, with $\Delta t_\Phi$ representing the common minimum response interval. The state transition can be calculated as follows:
\begin{align}
    s_{t+\Delta t_\Phi} = f_\Phi(s_t) = s_t \cup \bigcup_{i=1}^n \left( add(a_i) \setminus del(a_i) \right)
\end{align}
where $a_i$ is the action of robot $i$ in time $t$. If robot $i$ does not have an action or its action is running, we let $add(a_i) = del(a_i) = \emptyset$. 

The team region partition can be calculated as follows:
\begin{equation}
    r_\Phi(s) = 
\begin{cases}
    \text{\constant{R}} & \text{if } \exists i, r_i(s) = \text{\constant{R}} \\
    \text{\constant{S}} & \text{if } \forall i, r_i(s) \neq \text{\constant{R}} \text{ and } \exists i, r_i(s) = \text{\constant{S}} \\
    \text{\constant{F}} & \text{if } \forall i, r_i(s) = \text{\constant{F}}
\end{cases}
\end{equation}
The status of $\Phi$ is \constant{R} if any BT is still running, \constant{S} if some BT returns success and no one is running, and \constant{F} if all BT fails.
\begin{definition}[Finite Time Successful]
$\Phi$ is finite time successful (FTS) from region of attraction (ROA) $R$ to condition $c$, if $\forall s_0 \in R$  there is 
 a finite time $\tau$ such that for any $t<\tau$, $r_\Phi(s_t)=$ \constant{R}, and for any $t\geq\tau, r_\Phi(s_t)=$ \constant{S}, $c\in s_t$.
\end{definition}
With definitions above, the multi-robot BT planning problem can finally be defined.

\begin{problem}[Multi-Robot BT Planning]
The problem is a tuple \(\left<\mathcal{S},\mathcal{L},\{\mathcal{A}_i\}_{i=1}^n,\mathcal{M}, s_0,g\right>\), where \( \mathcal{S} \) is the finite set of environment states, $\mathcal{L}$ is the finite set literals that form states and conditions, $\mathcal{A}_i$ is the finite action set of robot $i$, $\mathcal{M}$ is the action model, $s_0$  is the initial state, $g$ is the goal condition. A solution to this problem is a BT set $\Phi = \{\mathcal{T}_i\}_{i=1}^n$ built with $\{\mathcal{A}_i\}_{i=1}^n$, such that $\Phi$ is FTS from $R\ni s_0$ to $g$.

\end{problem}

\begin{algorithm}[t]
\caption{One-step cross-tree expansion}
\label{alg:one_step}
\begin{algorithmic}[1]
\State \textbf{function} \constant{ExpandOneRobot(}$\mathcal{T},\mathcal{A}, c$\constant{)}
\begin{ALC@g}
\StateComment{$\mathcal{T}_{new} \leftarrow c$}{newly expanded subtree}
\StateComment{$\mathcal{C}_{new} \gets \emptyset$}{newly expanded conditions}
\FOR{\textbf{each} action $a \in \mathcal{A}$}
    \IF{$c \cap (pre(a) \cup add(a) \setminus del(a)) \neq \emptyset$ and $c \setminus del(a) = c$}
    \STATE $c_a \leftarrow pre(a) \cup c \setminus add(a)$
    \STATE $\mathcal{T}_a \leftarrow Sequence(c_a, a)$
    \STATE $\mathcal{T}_{new} \leftarrow Fallback(\mathcal{T}_{new}, \mathcal{T}_a)$
    \State{$\mathcal{C}_{new} \gets \mathcal{C}_{new} \cup \{c_a\}$}
    \ENDIF
\ENDFOR
\IF{$\mathcal{C}_{new}\neq \emptyset$}
    \IF{\constant{ConditionInTree}($c$, $\mathcal{T}$)}
        \StateComment{Replace $c$ with $\mathcal{T}_{new}$ in $\mathcal{T}$}{in-tree expand} \label{InsideExpand2}
    \ELSIF{$\mathcal{T}_{new} \neq c$} 
        \StateComment{$\mathcal{T} \gets Fallback(\mathcal{T},\mathcal{T}_{new}) $}{cross-tree expand}
    \ENDIF
\ENDIF
\end{ALC@g}
\RETURN $\mathcal{T}, \mathcal{C}_{new}$
\end{algorithmic}
\end{algorithm}

\section{Methods}

We first detail MRBTP, analyzing its soundness, completeness, and computational complexity. Then, we demonstrate how intention sharing functions among BTs during execution. Finally, we introduce the optional plugin, subtree pre-planning, to further enhance efficiency.

\subsection{Multi-Robot Behavior Tree Planning}

\subsubsection{One-Step Cross-Tree Expansion}
Algorithm \ref{alg:one_step} gives the pseudocode of one-step cross-tree expansion for one robot. Given its current BT $\mathcal{T}$, action space $\mathcal{A}$ and the condition to expand $c$ (line 1), the function returns an expanded BT $\mathcal{T}$ along with the newly expanded condition set $\mathcal{C}_{new}$ (line 15). Similar to one-step expansion for a single robot \cite{cai2021bt}, the expansion begins with $\mathcal{T}_{new}=c$ and $\mathcal{C}_{new}=\emptyset$ (line 2-3). Then we go through the action space $\mathcal{A}$ to find all premise actions that can lead to $c$ (line 4-5). For each premise action $a$, we calculate its corresponding precondition $c_a$ (line 6), form a sequence structure $\mathcal{T}_a$ (line 7), and add $\mathcal{T}_a$ to the tail of the root fallback node of $\mathcal{T}_{new}$ (line 8). Now $\mathcal{T}_{new}$ can achieve $c$ using these expanded actions if their precondition are met. We store these preconditions in $\mathcal{C}_{new}$ (line 9). 

If $\mathcal{T}_{new}$ is expanded (line 10), we need to decide where in $\mathcal{T}$ to place it. There are two cases: (1) $c$ is in $\mathcal{T}$, which means it was previously expanded by $\mathcal{T}$ itself. So we replace $c$ with $\mathcal{T}_{new}$ in $\mathcal{T}$ just like in single-robot BT expansion (line 11-12); (2) $c$ is not in $\mathcal{T}$, which means it was expanded by other BTs. To allow this BT to take actions to fulfill $c$, we add it to the tail of the root fallback node of $\mathcal{T}$ (line 13-14). 
\begin{proposition}\label{pro:inside}
    Given $\mathcal{T}$ is FTS from $R$ to $g$, if $\mathcal{T}$ is expanded by Algorithm \ref{alg:one_step} to $\mathcal{T}'$ given $c$, $c$ is in $\mathcal{T}$ and $\mathcal{C}_{new} \neq \emptyset$, then $\mathcal{T}'$ is FTS from $R'=R\cup \{ s\in \mathcal{S}| c_a \subseteq s, c_a\in \mathcal{C}_{new} \}$ to $g$.
\end{proposition}

\begin{proposition}\label{pro:outside}
    If $\mathcal{T}$ is expanded by Algorithm \ref{alg:one_step} to $\mathcal{T}'$ given $c$, $c$ is not in $\mathcal{T}$ and $\mathcal{C}_{new} \neq \emptyset$, then $\mathcal{T}'$ is FTS from $\mathcal{S}_{new}=\{ s\in \mathcal{S}| c_a \subseteq s, c_a\in \mathcal{C}_{new} \}$ to $c$.
\end{proposition}
The above two propositions state the changes in the ROA after one-step cross-tree expansion. If $c$ is in $\mathcal{T}$ (Proposition \ref{pro:inside}), the ROA of $\mathcal{T}$ will be expanded by $\mathcal{C}_{new}$ to achieve $g$. If $c$ is not in $\mathcal{T}$ (Proposition \ref{pro:outside}), then $c$ will be treated as a new sub-goal for $\mathcal{T}$ to be achieved from $\mathcal{S}_{new}$.\footnote{All the formal proofs are in the Appendix.}

\begin{algorithm}[t]
\caption{MRBTP}
\label{alg:MABTP}
\textbf{Input}: problem $\tuple{\mathcal{S},\mathcal{L},\multi{\mathcal{A}},\mathcal{M},s_0,c}$ \\
\textbf{Output}: solution $\Phi = \multi{\mathcal{T}}$ 
\begin{algorithmic}[1]

\STATE $\mathcal{C}_U\gets \{g\}$ \label{line:cu}
		\hfill $\triangleright$ conditions to be explored
\StateCommentLabel{$\mathcal{C}_E\gets \emptyset$}{expanded conditions}{line:ce}
\FOR{$i = 1$ \TO $n$}
    \STATE $\mathcal{T}_i \gets Fallback(g)$ \label{line:initBT} \Comment{init the BTs}
\ENDFOR
\WHILE{ $\mathcal{C}_U \neq \emptyset$}
\STATE $c\gets $ \constant{Pop(}$\mathcal{C}_U$\constant{)} \label{line:pickC}
\hfill $\triangleright$ explore $c$

\StateComment{\textbf{if} \algofunc{HasSubSet}{$c,\mathcal{C}_E$} \textbf{then} \textbf{continue}}{prune}  

\FOR{$i = 1$ \TO $n$}
    \STATE $\mathcal{T}_i, \mathcal{C}_{new} \gets$ \algofunc{ExpandOneRobot}{$\mathcal{T}_i,\mathcal{A}_i,c$} \label{line:initBT2}
    \IF{\algofunc{HasSubSet}{$s_0,\mathcal{C}_{new}$}}
        \RETURN $\Phi=\{\mathcal{T}_i\}_{i=1}^n$ \Comment{return a solution}
    \ENDIF
    \State{$\mathcal{C}_E \gets \mathcal{C}_E \cup \mathcal{C}_{new}$}
    \StateCommentLabel{$\mathcal{C}_U \gets \mathcal{C}_U \cup \mathcal{C}_{new}$}{add new conditions}{}
\ENDFOR

\ENDWHILE
  
\RETURN \constant{Unsolvable}
    
\end{algorithmic}
\end{algorithm}

\subsubsection{MRBTP}
Algorithm \ref{alg:MABTP} gives the pseudocode of MRBTP to plan BTs for the whole robot team. The algorithm initializes a set of conditions to be explored $\mathcal{C}_U=\{g\}$ and a set of expanded conditions $\mathcal{C}_E=\emptyset$ (line 1-2). The BT for each robot $i$ is initialized as $\mathcal{T}_i=Fallback(g)$ (line 4), which is FTS from $\emptyset$ to $g$. Then the algorithm continually explores conditions in $\mathcal{C}_U$ (line 5-6) until a solution is found, otherwise it returns \constant{Unsolvable} (line 14). For each explored $c$, it is either pruned if $\exists c' \in \mathcal{C}_E, c'\subseteq c$ (line 7), or expanded by all robots through one-step cross-tree expansion (line 8-9). After the one-step expansion for each robot, the newly expanded conditions $\mathcal{C}_{new}$ will be appended to $\mathcal{C}_E$ and $\mathcal{C}_U$ (line 12-13). If at that time $\exists c'\in \mathcal{C}_{new}, c'\subseteq s_0$, which means a solution is found, the algorithm returns $\Phi=\{\mathcal{T}_i\}_{i=1}^n$ as the solution (line 10-11).
\begin{proposition}\label{pro-k-loop2}: After the $k$-th ($k\geq1$) iteration of the while loop in Algorithm \ref{alg:MABTP}, where the explored condition is $c^{k}$, $\Phi^{k}=\multi{\mathcal{T}^{k}}$ is FTS from ROA $R^k=R^{k-1} \cup \bigcup_{i=1}^n \mathcal{S}^k_i$ to goal $g$, where $\mathcal{S}^{k}_i=\{s\in \mathcal{S}|c_a\subseteq s,c_a\in C^k_{i,new}\}$.
\end{proposition}
Note that Proposition \ref{pro-k-loop2} cannot be naturally extended from the single BT planning as it might seem. This proposition requires the assumption that robots execute in an appropriate order (at any time step, only the robot with the highest priority can execute an action if its precondition is satisfied); otherwise, deadlocks or departures from the ROA could occur. Fortunately, we can always use mechanisms such as deadlock detection during execution, ensuring that this serial execution is only employed in exceptional cases. In the vast majority of cases, robots can safely execute in parallel, so there is no need to worry that this assumption will reduce the execution efficiency of the robot team.

\begin{proposition}\label{pro-sound}
Algorithm \ref{alg:MABTP} is sound, i.e. if it returns a result $\Phi$ rather than \constant{Unsolvable}, then $\Phi$ is a solution of Problem 1.
\end{proposition}

\begin{figure*}[t]
	\centering
	\includegraphics[width=1\textwidth]{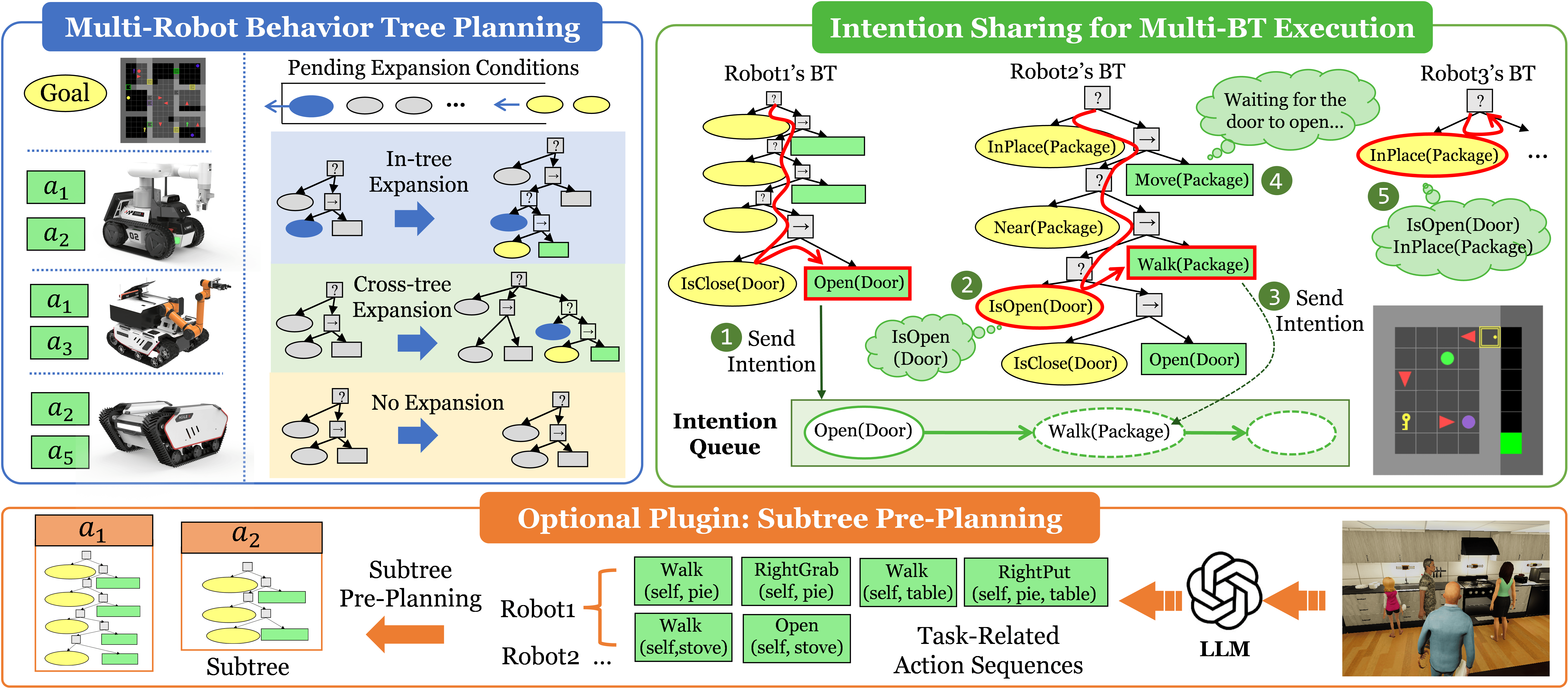}
	\caption{The framework of our paper. (1) MRBTP. A sound and complete algorithm for the multi-robot BT planning problem, capable of coordinating diverse actions across different BTs through cross-tree expansion. (2) Intention Sharing. Robots share intentions with each other during execution, enabling multi-BT parallelization without compromising failure tolerance. (3) Optional Plugin: Subtree Pre-planning. This plugin utilizes LLMs to pre-plan task-specific subtrees, establishing long-horizon action sequences to enhance MRBTP's planning and execution efficiency.}


	\label{fig:failure_example}
\end{figure*}

\begin{proposition}\label{pro-complete}
Algorithm \ref{alg:MABTP} is complete, i.e., if Problem 1 is solvable, the algorithm returns a $\Phi$ which is a solution.
\end{proposition}

Proposition \ref{pro-sound} can be proven by strong induction based on Proposition \ref{pro-k-loop}, and Proposition \ref{pro-complete} can be proven based on Proposition \ref{pro-sound}. These two propositions state the soundness and completeness of MRBTP, which makes it an effective algorithm to solve the multi-robot BT planning problem.

The time complexity of MRBTP in the worst case is $O(|\bigcup_{i=1}^n \mathcal{A}_i||\mathcal{S}||\mathcal{L}|)$, which is polynomial to the system size. In this case, the algorithm has to explore all states $s \in \mathcal{S}$ to find a solution. And in each exploration, the actions of all robots will be checked, with the checking complexity of $O(|\mathcal{L}|)$.

\subsection{Intention Sharing for Multi-BT Execution}

From the MRBTP planning process, we can observe that if multiple robots have identical actions (or similar actions with the same effect), MRBTP will expand them simultaneously in different BTs. This could lead to backup structures. These structures are beneficial for fault tolerance because if one robot fails, others can take over and complete the action. However, when multiple robots are available, backup structures can result in redundant execution. To avoid this, we introduce the multi-BT intention sharing method based on communication.




\paragraph{Intention Queue}
During execution, each robot $i$ maintains an intention queue $\mathcal{I}_i = (a_1, a_2, \dots, a_m)$ that indicates the actions being performed by other robots. In a situation with good communication, all robots' intention queues should remain consistent. Therefore, in the following text, we use $\mathcal{I}$ to refer to the intention queues of all robots. Based on the intention queue, we can calculate the belief success space $\mathcal{B}^S_i$ and the belief failure space $\mathcal{B}^F_i$ for robot $i$:
\setlength{\jot}{5pt}  
\begin{gather}
\mathcal{B}^S_{i} = \bigcup_{k=1}^{j-1} \left( add(a_k) \setminus del(a_k) \right) \\
\mathcal{B}^F_i = \bigcup_{k=1}^{j-1} \left( del(a_k) \setminus add(a_k) \right) 
\end{gather}
where $j$ is the index of its own action $a_j$ in the intention queue $\mathcal{I}$. If $j=1$, then $\mathcal{B}^S_i = \mathcal{B}^F_i = \emptyset$, which means the action is not dependent on any other's intention. If the robot currently has no action, it will be treated as $j = m+1$ when calculating belief spaces.

$\mathcal{B}^S$ and $\mathcal{B}^F$ will be used during the ticks of each BT. For each atomic condition node represented by a single literal $c = {l}$, it will first check if $l$ is in the belief spaces when ticked. If $l \in \mathcal{B}^S$, it returns \constant{S} without interacting with the environment, and returns \constant{F} when $l \in \mathcal{B}^F$. 

Whenever a robot $i$ exits an action or enters a new one, it will be broadcast to every other robot. Each robot then removes the old action of robot $i$ from the intention queue $\mathcal{I}$ (if it exists) and pushes the new action into it (if applicable). After this, each robot will update its belief spaces $\mathcal{B}^S$ and $\mathcal{B}^F$ to adjust its actions reactively. Note that an action exiting or entering may be due to two cases: (1) the environment state has changed, or (2) the belief spaces have changed. As a result, any addition or removal of actions in the intention queue $\mathcal{I}$ may lead to adjustments in other actions, creating a chain reaction. In other words, our intention-sharing method maintains the reactivity and robustness of BTs in response to uncertain environments.

\subsubsection{Parallelism and Blocking}

While intention sharing can avoid redundant execution, it also enhances action parallelism within the robot team. For example, as shown in Figure \ref{fig:failure_example}, in a warehouse management scenario, there are two robots capable of opening doors and transporting packages. They have expanded almost identical tree structures, sequentially executing \constant{Open}\constant{(Door)}, \constant{Walk}\constant{(Package)}, and \constant{Move}\constant{(Package)}. However, since \constant{IsClose(Door)}$\in s_0$, both robots satisfy the precondition to execute \constant{Open} \constant{(Door)}. Without intention sharing, they would perform this action simultaneously, causing redundancy. With intention sharing, however, if robot $1$ ticks its BT $\mathcal{T}_1$ first, it will execute \constant{Open(Door)} and send this intention to robot $2$. For robot $2$, \constant{IsOpen(Door)} $\in \mathcal{B}^S_2$ after updating the intention queue $\mathcal{I}$, so the corresponding condition node for \constant{IsOpen(Door)} will return \constant{S}, allowing the BT $\mathcal{T}_2$ to continue ticking and start executing \constant{Walk(Package)}. This transforms a serial BT structure into parallel execution.

However, when robot $2$ attempts to execute \constant{Move} \constant{(Package)}, which relies on the precondition \constant{IsOpen} \constant{(Door)}, robot $2$ will wait until the door is actually opened by robot $1$. Formally speaking, if $l\in \mathcal{B}^S_i$ but not in the current state $l\notin s$, when robot $i$ attempts to perform an action $a$ where $l\in pre(a)$, $a$ will be blocked. In this case, robot $i$ shares the intention of $a$, and $a$ returns \constant{R} as if it were executing, but it is actually doing nothing. The blocking mechanism prevents actions from being executed under incorrect preconditions, while also enabling the parallel execution of subsequent actions, thereby further enhancing the execution efficiency of the robot team.

\subsection{Optional Plugin: Subtree Pre-planning}

While MRBTP with intention sharing is proven to be an effective and efficient algorithm for the multi-robot BT planning problem, there is still room for further improving its efficiency. To achieve this, we begin by considering the following observations.
\begin{itemize}
\item During planning, the same tree structure might be generated multiple times, especially when multiple robots have overlapping action spaces.
\item During execution, sharing the intentions of every short-horizon atomic action not only increases the communication overhead but is also ineffective for long-term task scheduling.
\end{itemize}

A natural idea is that if we can obtain some long-horizon actions for each robot that are beneficial to the task, which we call subtrees, and add these actions to the corresponding robot's action space. During planning, we let these subtrees be prioritized over atomic actions, thereby speeding up the search for solutions and avoiding redundant planning. During execution, we only share the intentions of these subtrees. If the subtrees are well-designed, this approach can reduce communication overhead while also improving the efficiency of parallel execution.

\subsubsection{Subtree Pre-planning}

Let's first assume that we have obtained an action sequence $A=(a_1,a_2,\dots, a_m)$ for planning the subtree, and then consider how to use LLMs to generate task-related action sequences for each robot. Due to the modularity of the BT, we can treat the action sequence $A$ as a long-horizon action, and its action model can be calculated:
\setlength{\jot}{3pt}  
\begin{align}
pre(A) &= \bigcup_{j=1}^m \left(pre(a_j) \setminus \bigcup_{k=1}^j add(a_j)\right)\\
add(A) &= \bigcup_{j=1}^m \left( add(a_j) \setminus del(a_j) \right) - pre(A) \\
del(A) &= \bigcup_{j=1}^m \left( del(a_j) \setminus add(a_j) \right) 
\end{align}

We can obtain the tree structure of execution actions in $A$ sequentially by running a single-robot BT planning algorithm, with constraints on the order of actions to be expanded, a process we call subtree pre-planning.

However, to make a subtree behave like an atomic action, i.e., not to exit the precondition of $A$ while running in intermediate states, we need to introduce an additional subtree control structure, as illustrated in Figure \ref{fig-subtree_bt}. The subtree $\mathcal{T}_A$ has the preconditions \constant{Close(Door)} and \constant{Empty(Hand)}, but after \constant{Get(Key)}, the \constant{Empty(Hand)} condition is no longer satisfied. In the conventional BT planning algorithm, this would result in the subsequent actions not being ticked, causing the entire subtree $\mathcal{T}_A$ to fail. To address this issue, we introduce three subtree control nodes: \constant{EnterSubtree}, \constant{ExitSubtree}, and \constant{RunningSubtree}. If $pre(A)$ is satisfied and the robot is not currently running this subtree $\mathcal{T}_A$, then \constant{EnterSubtree} will be executed. This action will change the status of the subtree to running. The \constant{RunningSubtree} will return \constant{S} until \constant{ExitSubtree} is executed, or the BT begins executing a new action due to a change in the environment state. The parameter for the three nodes can be any identifier of the subtree. A simple way is to use the add effect $add(a_m)$ of the last action in the action sequence $A$ as the identifier.


\begin{figure}[t]
	\centering
	\includegraphics[width=0.47\textwidth]{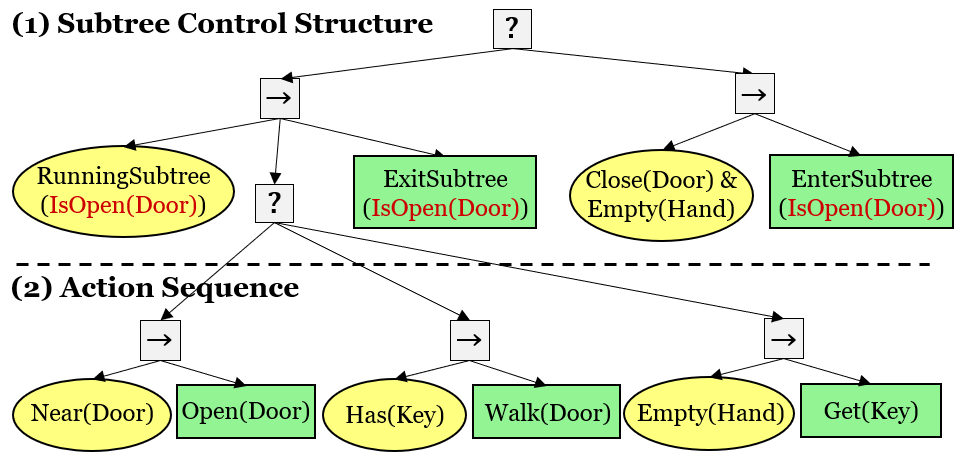}
	\caption{ An example of a pre-planned subtree structure for open a door.}
	\label{fig-subtree_bt}
\end{figure}

\begin{figure*}[ht]
    \centering
    \begin{subfigure}[b]{0.49\textwidth}
        \centering
        \includegraphics[width=\textwidth]{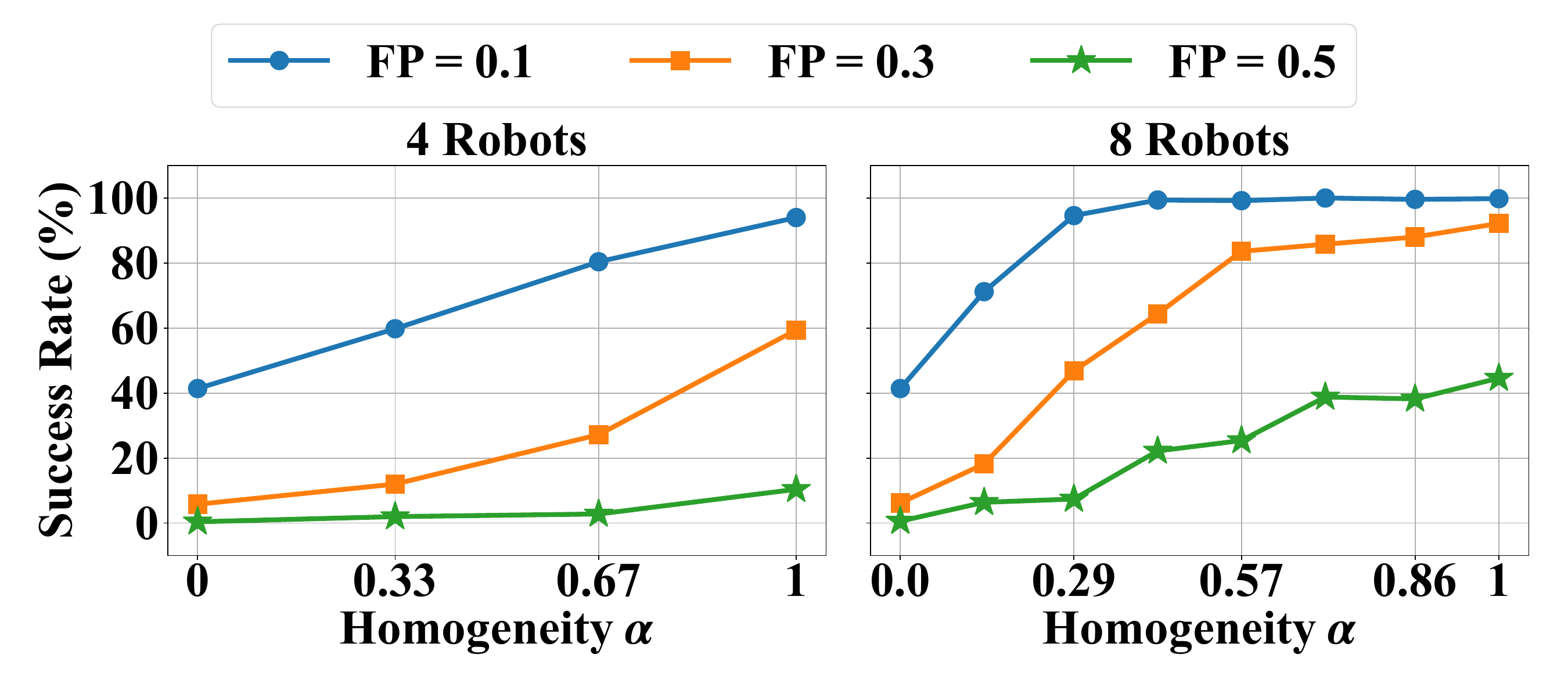}
        \caption{The impact of homogeneity on robustness}
        \label{fig:robust_results}
    \end{subfigure}
    \hfill
    \begin{subfigure}[b]{0.49\textwidth}
        \centering
        \includegraphics[width=\textwidth]{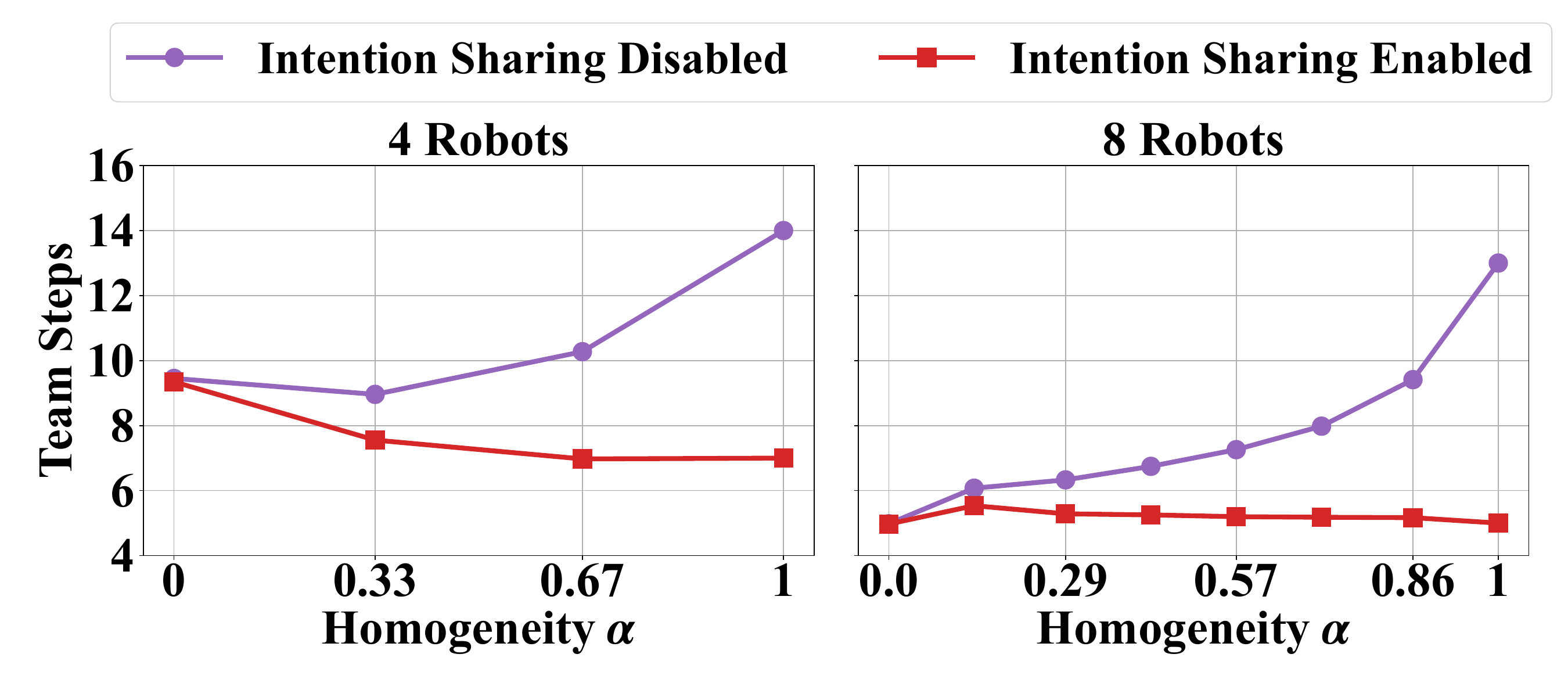}
        \caption{The impact of intention sharing on execution efficiency}
        \label{fig:subtask_chaining_results}
    \end{subfigure}
       \caption{Comparison of success rate and team steps under different conditions for 4 and 8 robots. Each data point represents the average of 500 trials.}
    \label{fig:combined_results}
\end{figure*}

\subsubsection{Prompt and Feedback for LLMs}
Appropriate subtrees for each robot can significantly improve the efficiency of planning and execution. However, obtaining these subtrees before planning is a very tricky task. Fortunately, pre-trained LLMs have been proven to possess task reasoning capabilities \cite{liu2023llm,song2023llmplanner}. When such models are available, subtree pre-planning can be a highly effective plugin for MRBTP.

To obtain suitable action sequences for each robot from pre-trained LLMs, the model's prompt should include: (1) task information, including the initial environment state and goal; (2) objects in the environment; (3) the action space for each robot; (4) few-shot demonstrations; and (5) a system prompt to guide the model to output correctly. 

After the LLM produces an output, we designed a checker to automatically verify it. We provide feedback to the LLM in three cases: (1) the output has grammar errors; (2) the action sequences cannot be pre-planned into a subtree, meaning they are not coherent; (3) the number of action sequences generated for each robot is insufficient. Once the output is good enough or the maximum number of feedback attempts has been reached, we begin subtree pre-planning. After pre-planning is completed, we add each subtree to the action spaces of all robots that contain all of the actions in the subtree, to fully utilize the subtrees.

\section{Experiments}

We evaluate the performance of MRBTP in two simulated scenarios: (1) Warehouse management with coarse action granularity and a smaller action space, and (2) Home service with finer granularity and a larger action space. First, we assess the robustness of the MRBTP method under varying levels of homogeneity by introducing a failure probability for each action. Next, we conduct an ablation study on intention sharing to verify its contribution to the execution efficiency of multi-robot BTs. Then, given the finer action granularity in the home service scenario, we perform an ablation study to evaluate subtree pre-planning, examining the effectiveness of pre-trained LLMs in generating task-related action sequences and their impact on the overall efficiency of the MRBTP. All experiments were conducted on a system equipped with an AMD Ryzen 9 5900X 12-core processor with a 3.70 GHz base clock and 128 GB of DDR4 RAM.

\subsection{Experimental Setup}

\subsubsection{Scenarios} (a) Warehouse Management. We extend the Minigrid \cite{chevalier-boisvert2023minigrid} environment for multi-robot simulations with 4-8 robots in 4 rooms containing randomly placed packages. Robots have diverse action spaces, including room inspection and package relocation, with some possessing specialized capabilities or restricted access. The goal is to optimize warehouse space utilization. (b) Home Service. In the VirtualHome \cite{puig2018virtualhome} environment, 2-4 robots interact with dozens of objects and perform hundreds of potential actions. Each robot’s action space is diverse, aiming to complete complex household tasks, such as setting the table or preparing a meal.

\subsubsection{Evaluation Metrics}
The algorithm's performance was evaluated using the following metrics: (a) Success Rate (SR): The percentage of successfully completed tasks across multiple trials, accounting for action failure probabilities. (b) Team Steps (TS): The total number of steps required for all robots to complete their tasks in parallel. (c) Total Robot Steps (RS): The sum of steps taken by each robot independently. (d) Communication Overhead (Comm.): The number of broadcast communications between robots due to intention sharing. (e) Number of Expanded Conditions (EC): The number of condition nodes expanded during the multi-robot BTs planning process, including those from subtree pre-planning if available. (f) Planning Time (PT): The time taken for multi-robot BT planning, including subtree pre-planning when available.

\subsubsection{Settings}
(a)  Homogeneity ($\boldsymbol{\alpha}$): The proportion of redundant actions assigned to robots, where $\boldsymbol{\alpha} = 1$ denotes complete heterogeneity (no overlap in action spaces) and $\boldsymbol{\alpha} = 0$ denotes complete homogeneity (identical action spaces).
(b) Action Failure Probability (FP): The probability that a robot fails to execute an action.
(c) Subtree Intention Sharing (Subtree IS) and Atomic Action Nodes Intention Sharing (Atomic IS): These terms refer to the application of Intention Sharing either among subtrees or at the level of individual atomic action nodes.
(d) Feedback (F) and No Feedback (NF): This setting distinguishes between LLMs that use feedback during subtree generation and those that do not. In the Feedback condition, the LLM receives up to 3 feedback iterations, while in the No Feedback condition, no feedback is provided.

\begin{table}[t]
\centering
\small
\setlength{\tabcolsep}{3pt}  
\begin{tabular}{c ccc c c c c}
\toprule
\multirow{2}{*}{\textbf{Method}} & \multicolumn{3}{c}{\textbf{$\boldsymbol{\alpha = 1}$ }} & & \textbf{\textbf{$\boldsymbol{\alpha \approx 0.5}$ }} & & \textbf{\textbf{$\boldsymbol{\alpha = 0}$ }} \\
\cline{2-4} \cline{6-6} \cline{8-8}  
\\[-1.5ex] 
&  \textbf{SR(\%)} & \textbf{TS} & \textbf{RS} & &  \textbf{SR(\%)} & &  \textbf{SR(\%)} \\
\midrule
\textbf{BT-Expansion} & 100 & 8.8 & 33.8 & & 12.4 & & 4.6 \\
\textbf{MRBTP} & 100 & 5.8 & 15.3 & & 100 & & 100 \\
\bottomrule
\end{tabular}
\caption{Performance comparison with baseline in warehouse management (4 robots, averaged over 500 trials).}
\label{tab:comparison_with_baseline}
\end{table}

\begin{table*}[t]
    \centering
    \small
    \setlength{\tabcolsep}{2pt} 
    \begin{tabular}{cc |cc cccc| cc cccc | cc cccc}
        \toprule
        \multicolumn{2}{c}{\textbf{Homogeneity}} & \multicolumn{6}{c}{\textbf{$\boldsymbol{\alpha = 1}$ }} & \multicolumn{6}{c}{\textbf{$\boldsymbol{\alpha \approx 0.5 } $}} & \multicolumn{6}{c}{\textbf{$\boldsymbol{\alpha=0}$}} \\

       \midrule 
      \multicolumn{2}{c|}{\textbf{Subtree}}   & - & - & \ding{51} &\ding{51} &\ding{51} &\ding{51} & - & - &  \ding{51} &\ding{51} &\ding{51} &\ding{51}  &- & -&  \ding{51} &\ding{51} &\ding{51} &\ding{51}  \\

        \multicolumn{2}{c|}{\textbf{Subtree IS}}          &- & - & -&-  & \ding{51} & \ding{51} &- & - & -&- & \ding{51} & \ding{51} &- & - &- &- & \ding{51} & \ding{51} \\ 
        \multicolumn{2}{c|}{\textbf{Atomic IS}}           &  - & \ding{51} & -  & \ding{51} & - & \ding{51}  & - &  \ding{51} & - & \ding{51} & - & \ding{51} & - & \ding{51} &- &\ding{51}  & -& \ding{51} \\

        \midrule
        \multirow{2}{*}{\textbf{TS}}   &  \textbf{NF}    & 161 & 159.4 & 114.4 & 109.6 & 78.9 & 79.8 & 139.7 & 137.5 & 126.2 & 119.6 & 86.9 & 102.5 & 73.7 & 68.5 & 96.1 & 94.8 & 75.6 & 78.0 \\ 
        
         &  \textbf{F} & - & - & 116.7 & 114.2 & \textbf{77.1} & 79.16 & - & - & 124.6 & 126.0 & \textbf{80.8} & 96.2 & - & - & 107.1 & 106.4 & \textbf{70.4} & 76.8 \\  
        \midrule

        \multirow{2}{*}{\textbf{RS}}   &  \textbf{NF}      & 570.8 & 557.3 & 374 & 359.4 & 217.4 & 219.1 & 385.2 & 380.1 & 345.8 & 326.6 & 209.6 & 222 & 128.6 & 128.6 & 128.6 & 128.6 & 128.6 & 128.6 \\ 
        &  \textbf{F}  & - & - & 377 & 370.9 & \textbf{205.2} & 208 & - & - & 380.7 & 348.2 & \textbf{192.2} & 209 & - & - & 128.6 & 128.6 & 128.6 & 128.6 \\ 

        \midrule
        \multirow{2}{*}{\textbf{Comm.}}  &  \textbf{NF}  & 0.0 & 63.8 & 0.0 & 7.1 & 6.7 & 14.1 & 0.0 & 43.5 & 0.0 & 8.0 & 7.1 & 20.9 & 0.0 & 15.2 & 0.0 & 2.8 & 6.5 & 9.3 \\ 
         & \textbf{F} & - & - & 0.0 & 4.8 & \textbf{6.6} & 12.2 & - & - & 0.0 & 4.4 & \textbf{6.4} & 13.2 & - & - & 0.0 & 0.7 & \textbf{5.2} & 6.0 \\

    \bottomrule
    \end{tabular}
    \caption{Execution efficiency with subtree pre-planning and intention sharing.}
    \label{tab:llm_sutree_result}
\end{table*}

\begin{table}[ht]
\centering
\small
\setlength{\tabcolsep}{5.5pt}  
\renewcommand{\arraystretch}{0.9} 
\begin{tabular}{ccccc}
\toprule
\textbf{Homogeneity} & \textbf{Subtree} & \textbf{Feedback} & \textbf{EC}   & \textbf{PT (s)}  \\ \midrule

\multirow{3}{*}{ \textbf{$\boldsymbol{\alpha = 1}$ } } & - & - & 8033.3  &  Timeout \\
 & \ding{51} & - & 998.1  & 12.4 \\
  & \ding{51} & \ding{51}& \textbf{384.3}  & \textbf{3.7} \\
\midrule

\multirow{3}{*}{\textbf{$\boldsymbol{\alpha \approx 0.5 } $} } & - & - & 7882.5 &  Timeout \\
 & \ding{51} & - &  623.8 & 7.2 \\ 
  & \ding{51} &\ding{51} & \textbf{267.9} & \textbf{2.6} \\ 

\midrule
\multirow{3}{*}{\textbf{$\boldsymbol{\alpha=0}$}} &- & - & 2695.5 & 20.2 \\
& \ding{51} & - &  576.6  & 5.6\\  
& \ding{51} & \ding{51} &  \textbf{146.8}  & \textbf{1.4}\\ 

\bottomrule
\end{tabular}
\caption{Planning efficiency with pre-planned subtrees. The average response time per LLM invocation is 4.2 seconds.}
\label{tab:planning_time}
\end{table}

\subsubsection{Baselines}

BT planning algorithms typically utilize action models for planning. To ensure consistency under the same problem assumptions, we propose directly adapting the BT-Expansion \cite{cai2021bt} algorithm, which has been proven sound and complete in single-robot settings, to multi-robot scenarios as our baseline. In BT-Expansion, each robot independently performs backward planning towards the team's goal, without incorporating cross-tree expansion or intention sharing.

\subsection{Experimental Results}

\subsubsection{Performance Comparison}
 
We randomly generated solvable multi-robot BT planning problems under various settings. Table \ref{tab:comparison_with_baseline} shows a significant drop in BT-Expansion's success rate as homogeneity decreases. In contrast, MRBTP maintains a perfect success rate of 100\% across all settings due to its cross-tree expansion. To avoid bias in execution efficiency (TS, RS) caused by planning failures, we only compared cases where both algorithms succeeded. Notably, even under full homogeneity, MRBTP outperforms BT-Expansion in execution efficiency due to intention sharing.

\subsubsection{Robustness}
As shown in Figure \ref{fig:robust_results}, the robustness of our algorithm improves with increasing homogeneity and is further enhanced by a larger number of robots. This improvement results from the increased likelihood of other robots compensating for action failures. Even with a 50\% failure probability per action, the system retains approximately a 50\% chance of achieving the goal with 8 robots and complete action space homogeneity.

\subsubsection{Execution Efficiency}

As shown in Figure \ref{fig:subtask_chaining_results}, in fully heterogeneous scenarios, enabling intention sharing results in fewer team steps, indicating that our MRBTP algorithm inherently maintains superior execution efficiency under these conditions. Additionally, as homogeneity increases, the likelihood of robots performing redundant actions rises, reducing the probability of parallel task execution. However, with intention sharing, redundant actions are significantly minimized, preventing further efficiency loss. In this context, increased homogeneity brings more backup structures, further improving execution efficiency.

\subsubsection{Effectiveness of Task-Specific Subtree Pre-Planning}
We constructed a dataset of 75 instances across three levels of homogeneity. The model used to generate subtrees is \textit{gpt-4o-mini-2024-07-18} \cite{openai2023gpt4}. Table \ref{tab:llm_sutree_result} provides a comparative analysis of the impact of introducing task-specific subtrees on execution efficiency, and communication overhead. The improvement in execution efficiency brought about by intention sharing increases with higher homogeneity. Execution efficiency is highest when subtrees are combined with intention sharing. During execution, the long-horizon subtrees facilitate forward planning, leading to more efficient and less frequent communication compared to finer-grained atomic actions. Table \ref{tab:planning_time} shows that subtree pre-planning significantly reduces BTs planning time under a 60-second constraint by minimizing redundancy through subtree reuse and similar robot action spaces. Additionally, both Table \ref{tab:llm_sutree_result} and Table \ref{tab:planning_time} demonstrate that feedback effectively enhances both planning and execution efficiency, especially when integrated with the sharing of the subtree and intentions.


\subsubsection{Execution Efficiency across Different LLMs}
We tested different versions of LLMs, including \textit{gpt-3.5-turbo (2024.12)} and \textit{gpt-4o-2024-08-06} \cite{openai2023gpt4}, for assisting in subtree pre-planning. As shown in Table \ref{tab:different_llms}, with the increased reasoning capability of the LLMs, there is a slight improvement in execution efficiency, while communication overhead remains largely unchanged. This can be attributed to the fact that subtree pre-planning becomes more appropriate and effective as the model's reasoning ability improves. Additionally, the results further demonstrate that the feedback mechanism enhances execution efficiency across all LLMs.

\begin{table}[t]
\centering
\small
\setlength{\tabcolsep}{3pt}  
\begin{tabular}{c ccc c ccc}
\toprule
\multirow{2}{*}{\textbf{Models}} & \multicolumn{3}{c}{\textbf{No Feedback}} &  & \multicolumn{3}{c}{\textbf{Feedback}} \\
\cline{2-4} \cline{6-8}
\\[-1.5ex] 
    & \textbf{TS} & \textbf{RS}  & \textbf{Comm.} & & \textbf{TS} & \textbf{RS}  & \textbf{Comm.}  \\
\midrule
\textbf{GPT-3.5-turbo}   & 81.6   & 223.6   & 5.1   & & 80.0   & 219.0   & 5.1   \\
\textbf{GPT-4o-mini}    & 78.9   & 217.4   & 6.7   & & 77.1   & 205.2   & 6.6   \\
\textbf{GPT-4o}    & 77.4   & 200.9   & 6.3   & & 74.9   & 190.7   & 6.3   \\ 
\bottomrule
\end{tabular}
\caption{Execution efficiency across different LLMs under $\boldsymbol{\alpha = 1}$ with subtree and subtree intention sharing.}
\label{tab:different_llms}
\end{table}

\section{Related Work}

\paragraph{BT Planning.}
Many works have focused on automatically generating BTs to perform tasks, such as evolutionary computing \cite{neupane2019learning,colledanchise2019learning,lim2010evolving}, reinforcement learning \cite{banerjee2018autonomous,pereira2015framework}, imitation learning \cite{french2019learning}, MCTS \cite{scheide2021behavior}, and formal synthesis \cite{li2021reactive,tadewos2022specificationguided,neupane2023designing}. Recently, some works directly generate BTs using LLMs \cite{lykov2023llmbrain,lykov2023llmmars}. However, the above methods either require complex environment modeling or cannot guarantee the reliability of BTs. In contrast, BT planning \cite{cai2021bt,chen2024integrating} based on STRIPS-style modeling \cite{fikes1971strips} not only offers intuitive environment modeling but also ensures the reliability and robustness of the generated BTs.

\paragraph{BT in MRS.}
BT generation for Multi-Robot Systems (MRS) has been investigated using various methodologies. Evolutionary computing \cite{neupane2019learning} is a general heuristic search method applied to BT generation in MRS. While versatile, this approach often suffers from slow search efficiency due to its lack of integration with the action model. Given the modular nature of BT systems, the action model is not difficult to obtain \cite{arora2018review}, enabling the development of methods that can yield more efficient solutions. MRS BT generation methods based on LLMs \cite{lykov2023llmmars} or other machine learning techniques \cite{fu2016reinforcement} have also been explored. These methods require substantial training data, making data collection and model training resource-intensive. Moreover, the aforementioned methods lack guarantees for the completeness and correctness of the generated BTs. Auction-based methods \cite{dahlquist2023reactive,heppner2024behavior,colledanchise2016advantages}, some of which incorporate action model planning, rely on the assumption of reliable communication and low transmission delay to ensure efficient task completion. However, such conditions are not always guaranteed, rendering these approaches less robust in environments with unreliable communication. In contrast, our method generates BTs before the robot team begins execution, ensuring task completion even in the absence of communication during execution. Communication during execution serves only to improve coordination efficiency, rather than being a necessary assumption.

\paragraph{LLM for Task Reasoning.} Recently, significant progress has been made in using LLMs for task reasoning \cite{song2023llmplanner,liu2023llm,ahn2022can,chen2023robogpt}, such as progprompt \cite{singh2022progprompt}, PlanBench \cite{valmeekam2023planbench}, and Voyager \cite{wang2023voyager}. Furthermore, the LLM has shown the ability to decompose the task into subgoals \cite{gao2024dagplan,singh2024twostep}, which is closely related to our subtree pre-planning for multi-robot BT planning. As the task reasoning abilities of LLMs continue to evolve and strengthen, our subtree pre-planning technique is poised to become increasingly relevant and effective.

\section{Conclusion}

We propose MRBTP, the first sound and complete algorithm for solving the multi-robot BT planning problem. The cross-tree expansion coordinates BTs for achieving goals, while intention sharing improves execution efficiency and robustness. The LLM plugin further enhances planning speed and reduces communication overhead. These contributions represent a key step forward in scalable, reliable multi-robot systems. Future research will refine the algorithm's performance and extend its application to more complex, dynamic environments, solidifying MRBTP as a foundational approach in multi-robot planning. Furthermore, the potential deployment of the algorithm on actual robotic systems will be explored, evaluating its effectiveness, scalability, and practicality in real-world scenarios.

\section{Acknowledgments}
This work was supported by the National Natural Science
Foundation of China (Grant Nos. 62106278, 62032024). 

\bigskip

\bibliography{aaai25}

\newpage
\newpage
\onecolumn
\appendix

\section{Appendix}

\section{A.Proofs of MRBTP}
In this section, we present a comprehensive version of the formalizations and proofs for Multi-Robot Behavior Tree Planning (MRBTP).

\subsection{Definitions}

\begin{definition}[Behavior Tree]
A behavior tree (BT) is a three-tuple $\mathcal{T} = <f, r, \Delta t>$. $f:2^{n}\rightarrow 2^{n}$ is its effect on the environment state, $\Delta t$ is the time step, and $r:2^{n}\mapsto \{ $\constant{S}, \constant{R}, \constant{F}\} partitions states into three regions, where $\mathcal{T}$ returns success, running, failure, respectively.
\end{definition}

In BT planning for a single robot \cite{cai2021bt}, we represent the problem as: \(<\mathcal{S},\mathcal{L},\mathcal{A},\mathcal{M}, s_0,g>\), where \( \mathcal{S} \) is the finite set of environment states, $\mathcal{L}$ is the finite set of literals that form states, \( \mathcal{A} \) is the finite set of actions, $\mathcal{M}$ is the action model, $s_0$  is the initial state, $g$ is the goal condition.

A condition $c$ in BT is usually a subset of a state $s$. If $c\subseteq s$, it is said condition $c$ holds in that state $s$. The state transition affected by action $a\in \mathcal{A}$ can be defined as a triplet \( \mathcal{M}(a)=<pre(a),add(a),del(a)> \), comprising the precondition, add effects, and delete effects of the action. We assume that an action always finishes in finite time.  If $a$ is finished after $k$ time step, the subsequent state $s_{t'}$ will be:
\begin{equation}\label{eqn-s_f}
	s_{t'}=f_a(s_t)=s_t\cup add(a)\setminus del(a), t'=t + k
\end{equation}
The following property holds for $\forall a \in \mathcal{A}$:
\begin{align}
	add(a) \cap del(a) = \emptyset \label{eqn:a_d}\\
	add(a) \cap pre(a) = \emptyset \label{eqn:a_p}
\end{align}

We then extend the BT representation from a single robot to a multi-robot system.

\begin{definition}[Multi-BT System]
A $n$-robot BT system is a four-tuple $\left<\Phi, f_\Phi, r_\Phi, \Delta t_\Phi\right>$, where $\Phi = \left\{ \mathcal{T}_i \right\}_{i=1}^n$ is the set of BTs, $f_\Phi: \mathcal{S} \mapsto \mathcal{S}$ is the team state transition function, $\Delta t_\Phi$ is the team time step, $r_\Phi: \mathcal{S} \mapsto \{$ \constant{S}, \constant{R}, \constant{F} $\}$ is the team region partition.
\end{definition}
Due to variability in hardware performance, we allow each robot's BT to have a different response frequency, with $\Delta t_\Phi$ representing the common minimum response interval. The state transition can be calculated as follows:
\begin{align}
    s_{t+\Delta t_\Phi} = f_\Phi(s_t) = s_t \cup \bigcup_{i=1}^n \left( add(a_i) \setminus del(a_i) \right)
\end{align}
where $a_i$ is the action of robot $i$ in time $t$. If robot $i$ do not have an action or its action is running, we let $add(a_i) = del(a_i) = \emptyset$. 

The team region partition can be calculated as follows:
\begin{equation}\label{equ-multiregion}
    r_\Phi(s) = 
\begin{cases}
    \text{\constant{R}} & \text{if } \exists i, r_i(s) = \text{\constant{R}} \\
    \text{\constant{S}} & \text{if } \forall i, r_i(s) \neq \text{\constant{R}} \text{ and } \exists i, r_i(s) = \text{\constant{S}} \\
    \text{\constant{F}} & \text{if } \forall i, r_i(s) = \text{\constant{F}}
\end{cases}
\end{equation}
The status of $\Phi$ is \constant{R} if any BT is still running, \constant{S} if some BT returns success and no one is running, and \constant{F} if all BT fails.
\begin{definition}[Finite Time Successful]
$\Phi$ is finite time successful (FTS) from region of attraction (ROA) $R$ to condition $c$, if $\forall s_0 \in R$  there is 
 a finite time $\tau$ such that for any $t<\tau$, $r_\Phi(s_t)=$ \constant{R}, and for any $t\geq\tau, r_\Phi(s_t)=$ \constant{S}, $c\subseteq s_t$.
\end{definition}
With definitions above, the multi-robot BT planning problem can finally be defined.

\begin{problem}[Multi-Robot BT Planning]\label{problem-MRBTP}
The problem is a tuple \(\left<\mathcal{S},\mathcal{L},\{\mathcal{A}_i\}_{i=1}^n,\mathcal{M}, s_0,g\right>\), where \( \mathcal{S} \) is the finite set of environment states, $\mathcal{L}$ is the finite set literals that form states and conditions, $\mathcal{A}_i$ is the finite action set of robot $i$, $\mathcal{M}$ is the action model, $s_0$  is the initial state, $g$ is the goal condition. A solution to this problem is a BT set $\Phi = \{\mathcal{T}_i\}_{i=1}^n$ built with $\{\mathcal{A}_i\}_{i=1}^n$, such that $\Phi$ is FTS from $R\ni s_0$ to $g$.

\end{problem}


\begin{algorithm}[t]
\caption{One-step cross-tree expansion}
\label{alg-one_step}
\begin{algorithmic}[1]
\State \textbf{function} \constant{ExpandOneRobot(}$\mathcal{T},\mathcal{A}, c$\constant{)}
\begin{ALC@g}
\StateComment{$\mathcal{T}_{new} \leftarrow c$}{newly expanded subtree}
\StateComment{$\mathcal{C}_{new} \gets \emptyset$}{newly expanded conditions}
\FOR{\textbf{each} action $a \in \mathcal{A}$}
    \IF{$c \cap (pre(a) \cup add(a) \setminus del(a)) \neq \emptyset$ and $c \setminus del(a) = c$}  \label{line-ifexpanded}
    \STATE $c_a \leftarrow pre(a) \cup c \setminus add(a)$ \label{line-ca}
    \STATE $\mathcal{T}_a \leftarrow Sequence(c_a, a)$ \label{line-ta}
    \STATE $\mathcal{T}_{new} \leftarrow Fallback(\mathcal{T}_{new}, \mathcal{T}_a)$ \label{line-tnew}
    \State{$\mathcal{C}_{new} \gets \mathcal{C}_{new} \cup \{c_a\}$}
    \ENDIF
\ENDFOR
\IF{$\mathcal{C}_{new}\neq \emptyset$}
    \IF{\constant{ConditionInTree}($c$, $\mathcal{T}$)}
        \StateComment{Replace $c$ with $\mathcal{T}_{new}$ in $\mathcal{T}$}{in-tree expand} \label{InsideExpand}
    \ELSIF{$\mathcal{T}_{new} \neq c$} 
        \StateComment{$\mathcal{T} \gets Fallback(\mathcal{T},\mathcal{T}_{new}) $}{cross-tree expand} \label{line-outside}
    \ENDIF
\ENDIF
\end{ALC@g}
\RETURN $\mathcal{T}, \mathcal{C}_{new}$
\end{algorithmic}
\end{algorithm}

\subsection{Propositions and Proofs}

\begin{lemma}\label{lem-ta}
Given a condition $c$, the sequence structure $\mathcal{T}_a = sequence(c_a,a)$ expanded in Algorithm \ref{alg-one_step} (line \ref{line-ta}) is FTS from $\mathcal{S}_a=\{s\in \mathcal{S} | c_a\subseteq s\}$ to $c$.
\end{lemma}

\begin{proof}
Starting from any $s_t\in \mathcal{S}_{a}=\{s\in \mathcal{S}| c_a \subseteq s\}$, $\exists c_a, c_a\subseteq s_t$. According to Equation \ref{eqn-s_f}, there exists a finite $k$ such that the action returns success and $s_{t+k}=s_t\cup add(a) \setminus del(a)\supseteq c_a \cup add(a) \setminus del(a)=pre(a)\cup c \setminus del(a)$. Since action selection (line \ref{line-ifexpanded}) ensures that $c\setminus del(a) = c$, we have $s_{t+k}\supseteq pre(a)\cup c \supseteq c$. Therefore $\mathcal{T}_a$ is FTS from $\mathcal{S}_{a}$ to $c$.
\end{proof}

\begin{lemma}\label{lem-tnew}
Given a condition $c$ and $\mathcal{C}_{new}\neq \emptyset$, $\mathcal{T}_{new}$ expanded in Algorithm \ref{alg-one_step} (line \ref{line-tnew}) is FTS from $\mathcal{S}_{new}=\{s\in \mathcal{S} | c_a\subseteq s, c_a\in \mathcal{C}_{new}\}$ to $c$.
\end{lemma}

\begin{proof}
Since $\mathcal{C}_{new}\neq\emptyset$, $\mathcal{T}_{new}$ has and only has one fallback node as the root, given $Fallback(\mathcal{T}_1,Fallback(\mathcal{T}_2,\mathcal{T}_3)) =
Fallback(\mathcal{T}_1,\mathcal{T}_2,\mathcal{T}_3)$. The fallback node returns success as long as any child returns success. According to Lemma \ref{lem-ta}, $\forall c_a\in \mathcal{C}_{new}$, the corresponding $\mathcal{T}_a$ is FTS from $\mathcal{S}_a$ to $c$, and $\mathcal{C}_{new}$ is a finite set as $\mathcal{A}$ is a finite set.  Therefore $\mathcal{T}_{new}$ is FTS from $\mathcal{S}_{new}=\bigcup_{c_a\in \mathcal{C}_{new}} \mathcal{S}_a = \{s\in \mathcal{S} | c_a\subseteq s, c_a\in \mathcal{C}_{new}\}$ to $c$. 
\end{proof}

\begin{proposition}\label{pro-inside}
    Given $\mathcal{T}$ is FTS from $R$ to $g$, if $\mathcal{T}$ is expanded by Algorithm \ref{alg-one_step} to $\mathcal{T}'$ given $c$, $c$ is in $\mathcal{T}$ and $\mathcal{C}_{new} \neq \emptyset$, then $\mathcal{T}'$ is FTS from $R'=R\cup \{ s\in \mathcal{S}| c_a \subseteq s, c_a\in \mathcal{C}_{new} \}$ to $g$.
\end{proposition}

\begin{proof}

For any $s\in R'$, we consider two cases:

(1) $s\in R$. In this case, either $\mathcal{T}_{new}$ in the BT $\mathcal{T}'$ is not ticked ($c\nsubseteq s$), or $\mathcal{T}_{new}$ returns success ($c\subseteq s$). In either situation, the execution logic of the BT is the same as that of the original $\mathcal{T}$. Since $\mathcal{T}$ is FTS from $R$ to $g$, the expanded $\mathcal{T}'$ is also FTS from $R$ to $g$.

(2) $s\notin R$ and $s\in \mathcal{S}_{new}=\{ s\in \mathcal{S}| c_a \subseteq s, c_a\in \mathcal{C}_{new} \}$. In this case, the expended $\mathcal{T}_{new}$ in BT $\mathcal{T}'$ will be ticked. According to Lemma \ref{lem-tnew}, $\mathcal{T}_{new}$ is FTS from $\mathcal{S}_{new}$ to $c$, which means there is a finite time $\tau_1$ such that for any $t<\tau_1$, $r_{\mathcal{T}'}(s_t)=$\constant{R}, and for $t=\tau_1$, $r_{\mathcal{T}_{new}}(s_{\tau_1})=$\constant{S}, $c\subseteq s_{\tau_1}$. Since $s_{\tau_1}\in R$, there is a finite time $t=\tau_2$ such that for any $t<\tau_1+\tau_2$, 
$r_{\mathcal{T}'}(s_t)=$\constant{R}, and for any $t\geq \tau_1+\tau_2$, $r_{\mathcal{T}'}(s_t)=$\constant{S}, $g\subseteq s_t$. Therefore $\mathcal{T}'$ is FTS from $\mathcal{S}_{new}\setminus R$ to $g$.

According (1) and (2), $\mathcal{T}'$ is FTS from $R'= R\cup \{ s\in \mathcal{S}| c_a \subseteq s, c_a\in \mathcal{C}_{new} \}$ to $g$.
\end{proof}

\begin{proposition}\label{pro-outside}
    If $\mathcal{T}$ is expanded by Algorithm \ref{alg-one_step} to $\mathcal{T}'$ given $c$, $c$ is not in $\mathcal{T}$ and $\mathcal{C}_{new} \neq \emptyset$, then $\mathcal{T}'$ is FTS from $S_{new}=\{ s\in \mathcal{S}| c_a \subseteq s, c_a\in \mathcal{C}_{new} \}$ to $c$.
\end{proposition}

\begin{proof}
Because $c$ is not in $\mathcal{T}$, the expanded subtree $\mathcal{T}_{new}$ is added to the tail of the root fallback node of $\mathcal{T}$ (line \ref{line-outside}). $\forall s \in \mathcal{S}_{new}$, $\mathcal{T}_{new}$ in $\mathcal{T}'$ will be ticked. According to Lemma \ref{lem-tnew}, $\mathcal{T}_{new}$ is FTS from $\mathcal{S}_{new}$ to $c$, which means $\mathcal{T}'$ is also FTS from $\mathcal{S}_{new}$ to $c$.
\end{proof}



\begin{algorithm}[t]
\caption{MABTP}
\label{alg:MABTP}
\textbf{Input}: problem $\tuple{\mathcal{S},\mathcal{L},\multi{\mathcal{A}},\mathcal{M},s_0,c}$ \\
\textbf{Output}: solution $\Phi = \multi{\mathcal{T}}$ 
\begin{algorithmic}[1]

\STATE $\mathcal{C}_U\gets \{g\}$ \label{line:cu}
		\hfill $\triangleright$ conditions to be explored
\StateCommentLabel{$\mathcal{C}_E\gets \emptyset$}{expanded conditions}{line:ce}
\FOR{$i = 1$ \TO $n$}
    \STATE $\mathcal{T}_i \gets Fallback(g)$ \label{line-initBT} \Comment{init the BTs}
\ENDFOR
\WHILE{ $\mathcal{C}_U \neq \emptyset$}
\STATE $c\gets $ \constant{Pop(}$\mathcal{C}_U$\constant{)} \label{line:pickC}
\hfill $\triangleright$ explore $c$

\StateComment{\textbf{if} \algofunc{HasSubSet}{$c,\mathcal{C}_E$} \textbf{then} \textbf{continue}}{prune}  

\FOR{$i = 1$ \TO $n$}
    \STATE $\mathcal{T}_i, \mathcal{C}_{new} \gets$ \algofunc{ExpandOneRobot}{$\mathcal{T}_i,\mathcal{A}_i,c$} \label{line:initBT}
    \IF{\algofunc{HasSubSet}{$s_0,\mathcal{C}_{new}$}} \label{line-returnif}
        \RETURN $\Phi=\{\mathcal{T}_i\}_{i=1}^n$ \Comment{return a solution} \label{line-return}
    \ENDIF
    \State{$\mathcal{C}_E \gets \mathcal{C}_E \cup \mathcal{C}_{new}$}
    \StateCommentLabel{$\mathcal{C}_U \gets \mathcal{C}_U \cup \mathcal{C}_{new}$}{add new conditions}{}
\ENDFOR

\ENDWHILE
  
\RETURN \constant{Unsolvable} \label{line-returnu}
    
\end{algorithmic}
\end{algorithm}

\begin{lemma}\label{lem-treeonestep}
After the one-step expansion of the multi-BT system \( \Phi = \{\mathcal{T}_i\}_{i=1}^{n} \) with respect to condition \( c \), let \( \Phi_{new} = \{\mathcal{T}_{i,new}\}_{i=1}^{n} \) represent the set of all extended subtrees. $\Phi_{new}$ is FTS from $\bigcup_{i=1}^n \mathcal{S}_{i,new}$ to $c$, where $\mathcal{S}_{i,new}=\{s\in \mathcal{S}|c_a\subseteq s,c_a\in C_{i,new}\}$.
\end{lemma}

\begin{proof}
According to Lemma \ref{lem-tnew}, $\forall \mathcal{T}_{i,new} \in \Phi_{new}$, $\mathcal{T}_{i,new}$ is FTS from $\mathcal{S}_{i,new}$ to $c$. Therefore, for any $s\in \bigcup_{i=1}^n \mathcal{S}_{i,new}$, at least one subtree can run in $s$. Since we assume that at each time, only one robot with the highest priority can execute an action if its precondition is satisfied, we designate the executed subtree as $\mathcal{T}_{j,new}$. Because $\mathcal{T}_{j,new}$ is FTS to $c$, we have that $\Phi_{new}$ is FTS from $\bigcup_{i=1}^n \mathcal{S}_{i,new}$ to $c$.
\end{proof}

\begin{proposition}\label{pro-k-loop}: After the $k$-th ($k\geq 1$) iteration of the while loop in Algorithm \ref{alg:MABTP}, where the explored condition is $c^{k}$, $\Phi^{k}=\multi{\mathcal{T}^{k}}$ is FTS from ROA $R^k=R^{k-1} \cup \bigcup_{i=1}^n \mathcal{S}^k_i$ to goal $g$, where $\mathcal{S}^{k}_i=\{s\in \mathcal{S}|c_a\subseteq s,c_a\in C^k_{i,new}\}$.
\end{proposition}

\begin{proof}
This proposition can be proved by strong induction. In the basis step (before the first expansion), all the BT $\mathcal{T}^0_i\in \Phi^0$ is $Fallback(g)$ (line \ref{line-initBT}). According to \ref{equ-multiregion}, $\Phi^0=\multio{\mathcal{T}^0_i}$ if FTS from $R^0=\{s\in \mathcal{S}|g\subseteq s\}$ to $g$.

For the inductive step, we assume after the $k$-th ($k\geq 1$) iteration of the while loop, $\Phi^k$ is FTS from $R^k$ to $g$. Then after the $(k+1)$-th iteration, for $s\in R^{k} \cup \bigcup_{i=1}^n \mathcal{S}^{k+1}_i$, there are two cases:

(1) $s\in R^{k}$. In this case, $\Phi^{k+1}$ has the same execution logic as $\Phi^{k}$, just similar to case (1) in Proposition \ref{pro-inside}. Therefore $\Phi^{k+1}$ is FTS from $R^{k}$ to $g$.

(2) $s\notin R^{k}$ and $s\in \bigcup_{i=1}^n \mathcal{S}^{k+1}_i$. In this case, one of the expended $\mathcal{T}^{k+1}_{i,new}\in\Phi^{k+1}_{new}$ will be ticked and the state will go to some $s'\supseteq c$ in finite time according to Lemma \ref{lem-treeonestep}. Since $s'\in R^{k}$, $\Phi^{k+1}$ will succeed to $g$ in finite time according to (1).

According (1) and (2), after the $k$-th interation, $\Phi^{k}$ is FTS from ROA $R^k=R^{k-1} \cup \bigcup_{i=1}^n \mathcal{S}^k_i$ to goal $g$.
\end{proof}

\begin{proposition}\label{pro-sound}
Algorithm \ref{alg:MABTP2} is sound, i.e. if it returns a result $\Phi$ rather than \constant{Unsolvable}, then $\Phi$ is a solution of Problem 1.
\end{proposition}

\begin{proof}
If Algorithm \ref{alg:MABTP2} returns a result $\Phi^k$ after $k$-th iteration, that means $\exists c\in \mathcal{C}^k_{new}, c\subseteq s_0$ (line \ref{line-returnif}). So $s_0$ is in the ROA of $\Phi^k$: $R^k=R^{k-1} \cup \bigcup_{i=1}^n \mathcal{S}^k_i$. According to Proposition \ref{pro-k-loop}, $\Phi^k$ if FTS from $R^k$ to $g$. That makes $\Phi^k$ a solution of Problem \ref{problem-MRBTP}.
\end{proof}

\begin{lemma}\label{lemma_terminate}
    Algorithm \ref{alg:MABTP2} terminates in finite time.
\end{lemma}

\begin{proof}
This is because the totoal number of states $|\mathcal{S}|$ is finite, therefore the number of conditions is also finite. The algorithm either returns prematurely (line \ref{line-return}) or terminates when \( \mathcal{C}_U = \emptyset \) (line \ref{line-returnu}). Within the while loop, the loop for one-step expansion also terminates in finite time because the total number of actions $|\multi{\mathcal{A}}|$ is finite. Therefore, Algorithm \ref{alg:MABTP2} terminates in finite time.
\end{proof}

\begin{proposition}\label{pro-complete}
Algorithm \ref{alg:MABTP2} is complete, i.e., if Problem \ref{problem-MRBTP} is solvable, the algorithm returns a $\Phi$ which is a solution.
\end{proposition}

\begin{proof}
According to Lemma \ref{lemma_terminate}, Algorithm \ref{alg:MABTP2} either returns prematurely or terminates when $\mathcal{C}_U=\emptyset$ in finite time.

If it returns a $\Phi$ prematurely, according to Proposition \ref{pro-sound}, $\Phi$ is a solution of Problem \ref{problem-MRBTP}.

If it terminates when $\mathcal{C}_U=\emptyset$, it implies that all conditions reachable from the goal condition \( g \) have been explored, but no viable path from the goal \( g \) to the initial state \( s_0 \) has been found. This suggests that Problem \ref{problem-MRBTP} is not solvable. 

This can be proved by contradiction. We assume that there exists a solution $\Phi$ but our algorithm fails. $\Phi$ is FTS from $R\ni s_0$ to $g$, so we can utilize a finite sequence to represent the state transition of the solution BT: $(s_0,a_1,s_1,a_2,\dots,s_{m-1},a_m,s_m)$ where the goal condition $g\subseteq s_m$. We can prove, via mathematical induction from $s_m$ to $s_0$ in this sequence, that every state in this sequence must satisfy some condition that was traversed during the while loop.

For the basis step, we consider the state $s_{m-k}$ with $k=0$, i.e. $s_m\supseteq g$. Obviously $g$ is in all BTs of $\Phi$.

For the inductive step, we consider a transition $(s_{m-k-1},a_{m-k},s_{m-k})$ and there exists a condition $c\subseteq s_{m-k}$ in some BT $\mathcal{T}_i\in \Phi$ as the inductive premise. We need to prove that there also exists a condition satisfied by $s_{m-k-1}$. The inductive premise and the transition gives: 
\begin{equation}\label{eqn:snk1}
s_{m-k-1}\cup add(a_{m-k})\setminus del(a_{m-k}) = s_{m-k}\supseteq c
\end{equation}
which can be deduced to:
\begin{equation}\label{eqn:snk1s}
s_{m-k-1}\supseteq c\setminus add(a_{m-k})
\end{equation}

We then analyze in two cases, whether $c \cap (pre(a_{m-k})\cup add(a_{m-k}) \setminus del(a_{m-k}))=\emptyset$.

(1) $c \cap (pre(a_{m-k})\cup add(a_{m-k}) \setminus del(a_{m-k}))=\emptyset$: From the case premise, $c\cap add(a_{m-k}) = \emptyset$ holds because the add effects and delete effects has no common member (Equation \ref{eqn:a_d}). Therefore Equation \ref{eqn:snk1s} can further be deduced to $s_{m-k-1}\supseteq c$, which means $s_{m-k-1}$ satisfies the condition $c$ of the BT.

(2) $c \cap (pre(a_{m-k})\cup add(a_{m-k}) \setminus del(a_{m-k}))\neq\emptyset$: The transition shown by eqn. \ref{eqn:snk1} indicate  $s_{m-k}\setminus del(a_{m-k})=s_{m-k}$. Therefore when expanding $c$, action $a_{m-k}$ will be explored because the selective condition holds (line \ref{line-ifexpanded}), which creates a condition node $c_{a_{m-k}}=pre(a_{m-k})\cup c \setminus add(a_{m-k})$. The state transition also provides that $s_{m-k-1}\supseteq pre(a_{m-k})$, then with Equation \ref{eqn:snk1s} we can deduce:
\begin{equation}\label{eqn:snk1sp}
	s_{m-k-1}\supseteq pre(a_{m-k}) \cup c\setminus add(a_{m-k})
\end{equation}
Therefore $s_{m-k-1}$ satisfies the condition $c_{a_{m-k}}$. Assume $a_{m-k}$ is in the action space of robot $i$, $a_{m-k}\in \mathcal{A}_i$, then $c_{a_{m-k}}$ will be expanded to its BT $\mathcal{T}_i\in \Phi$.

The two cases complete the proof of the inductive step.

The mathematical induction proves that any state in the sequence must satisfy some condition node expanded in the while-loop. Therefore $s_0$ will satisfy some condition
node $c\in \mathcal{C}_{new}$ after one-step expansion, then the algorithm will return a $\Phi$, which contradicts the initial assumption that a solution $\Phi$ exists but our algorithm fails. 

Therefore if Problem \ref{problem-MRBTP} is solvable, the algorithm returns a solution $\Phi$.
\end{proof}

\section{B.Limitations and Future Work}

\subsection{Assumption}

As mentioned above, we assume that at each time, only one robot with the highest priority can execute an action if its precondition is satisfied. Otherwise, two exceptions may occur:

(1) Deadlocks. For example, at time $t$, the robot $1$ takes an action $a_1$ that changes the state from $s_1$ to $s_2$, but the robot $2$ then takes an action $a_2$ and changes $s_2$ back to $s_1$. This loop will continue indefinitely.

(2) Departures from the ROA. For example, the state is $\{1,2\}$ at time $t$, the robot $1$ attempts to take $a_1$ and the robot $2$ attempts to take $a_2$. The action models are:
\begin{align}
pre(a_1)=\{1,2\}, add(a_1)=\{3\}, del(a_1)=\{1\} \nonumber \\
pre(a_2)=\{1,2\}, add(a_2)=\{4\}, del(a_2)=\{2\} \nonumber
\end{align}

If only $a_1$ is applied, the state will transfer to $\{2,3\}$. If only $a_2$ is applied, the state will transfer to $\{1,4\}$. Both $\{2,3\}$ and $\{1,4\}$ is in the ROA of $\Phi$. However, if $a_1$ and $a_2$ are applied simultaneously, the state will transfer to $\{3,4\}$, which is out of the ROA.

As we demonstrated in the paper, we can always use mechanisms such as deadlock detection and intention sharing during execution, ensuring that this serial execution is only employed in exceptional cases. In the vast majority of cases, robots can execute in parallel safely, so there is no need to worry that this assumption will reduce the efficiency of the robot team.

Nevertheless, we believe it is valuable to explore in future work whether it is possible to design an improved algorithm that addresses the aforementioned issues during the planning process.

\subsection{Optimality}
MRBTP can guarantee finding a solution for Problem \ref{problem-MRBTP}; however, since the cost of actions is not considered, MRBTP does not account for the optimality of the solution obtained. In the future, designing an algorithm for solving the optimal multi-robot BT planning problem is also a valuable research field.

\subsection{Parallelization}
MRBTP can achieve a certain degree of multi-robot parallel execution through intention sharing during execution. However, this parallelism is not carefully scheduled. In the future, we may explore the development of an algorithm that can automatically identify parallelizable subtrees and combine it with a runtime scheduling algorithm to achieve more efficient multi-robot parallelism.

\subsection{Subtree Generation}
Currently, we rely on the task reasoning capabilities of LLMs to generate task-specific subtrees. This approach becomes ineffective when LLMs are unavailable. Therefore, we might explore the development of an efficient subtree pre-planning algorithm that does not depend on LLMs, to quickly obtain useful subtree structures before the formal planning process of MRBTP.

\section{C.Experiments}
\subsection{Computing Infrastructure}
All experiments were conducted on a system equipped with an AMD Ryzen 9 5900X 12-Core Processor (3.70 GHz base clock), 128 GB DDR4 RAM, and running Windows 10, 64-bit.



\subsection{Simulation Scenarios}

The experimental design encompasses two distinct scenarios, each with two simulation modes: computational and scenario simulations. In the Warehouse Management scenario, up to 8 robots are tasked with organizing supplies across 5 designated rooms. This setup utilizes two action predicates, \constant{OpenRoom} and \constant{MovePackage}, and four condition predicates: \constant{IsOpen}, \constant{IsClose}, \constant{IsHolding}, and \constant{IsInRoom}. Similarly, the Home Service scenario involves up to 4 robots assigned to household tasks, employing the same set of action and condition predicates across both simulation modes. Figure \ref{fig:scenario} depicts the two simulation scenarios, and Table \ref{table:action_predicates} details the action and condition predicates used in the scenario simulation mode.

\begin{figure}[h]
    \centering
    \begin{subfigure}[b]{0.47\linewidth}
        \includegraphics[width=\linewidth]{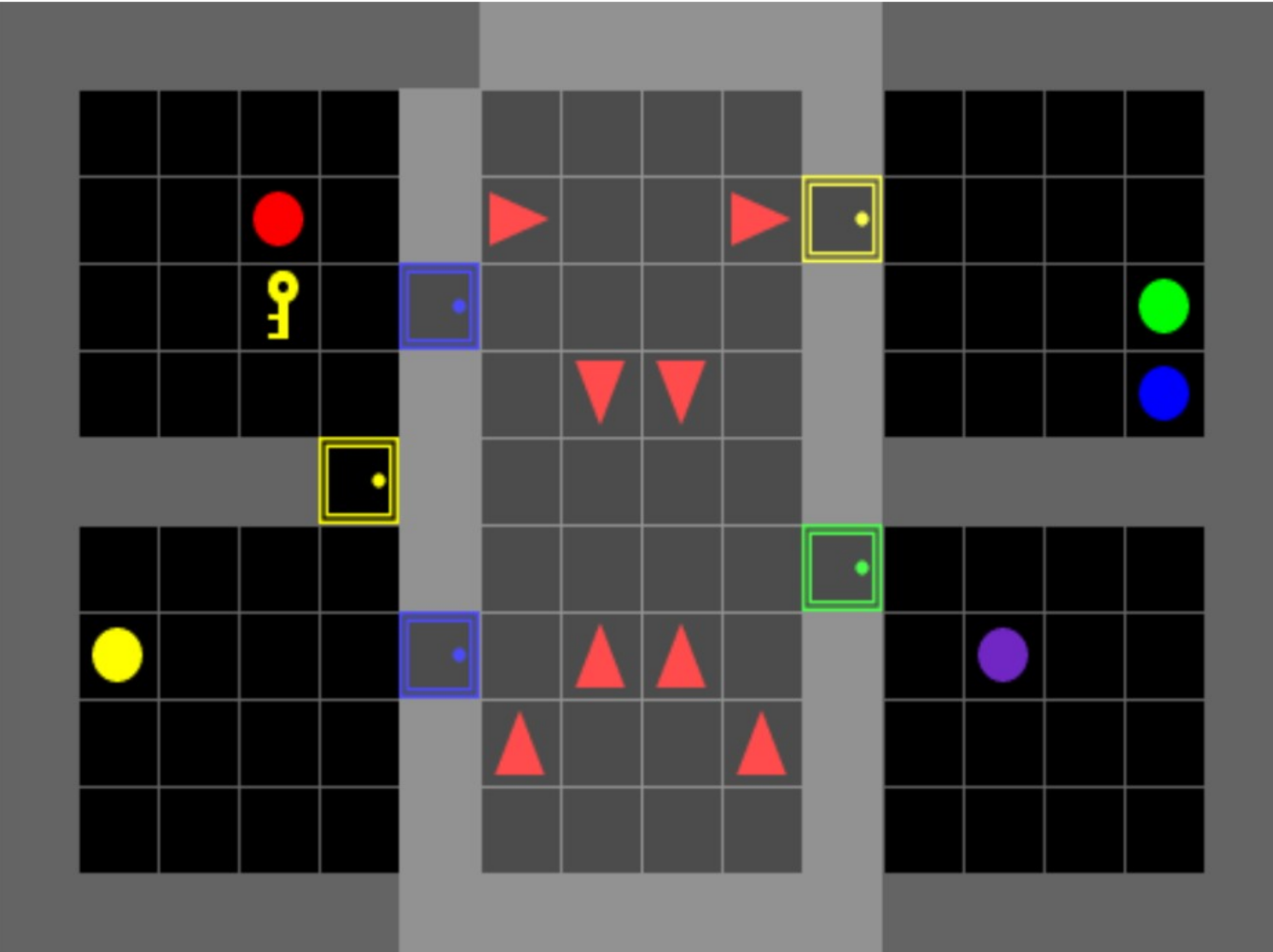}
        \caption{Warehouse Management}
        \label{fig:warehouse}
    \end{subfigure}
    \begin{subfigure}[b]{0.50\linewidth}
        \includegraphics[width=\linewidth]{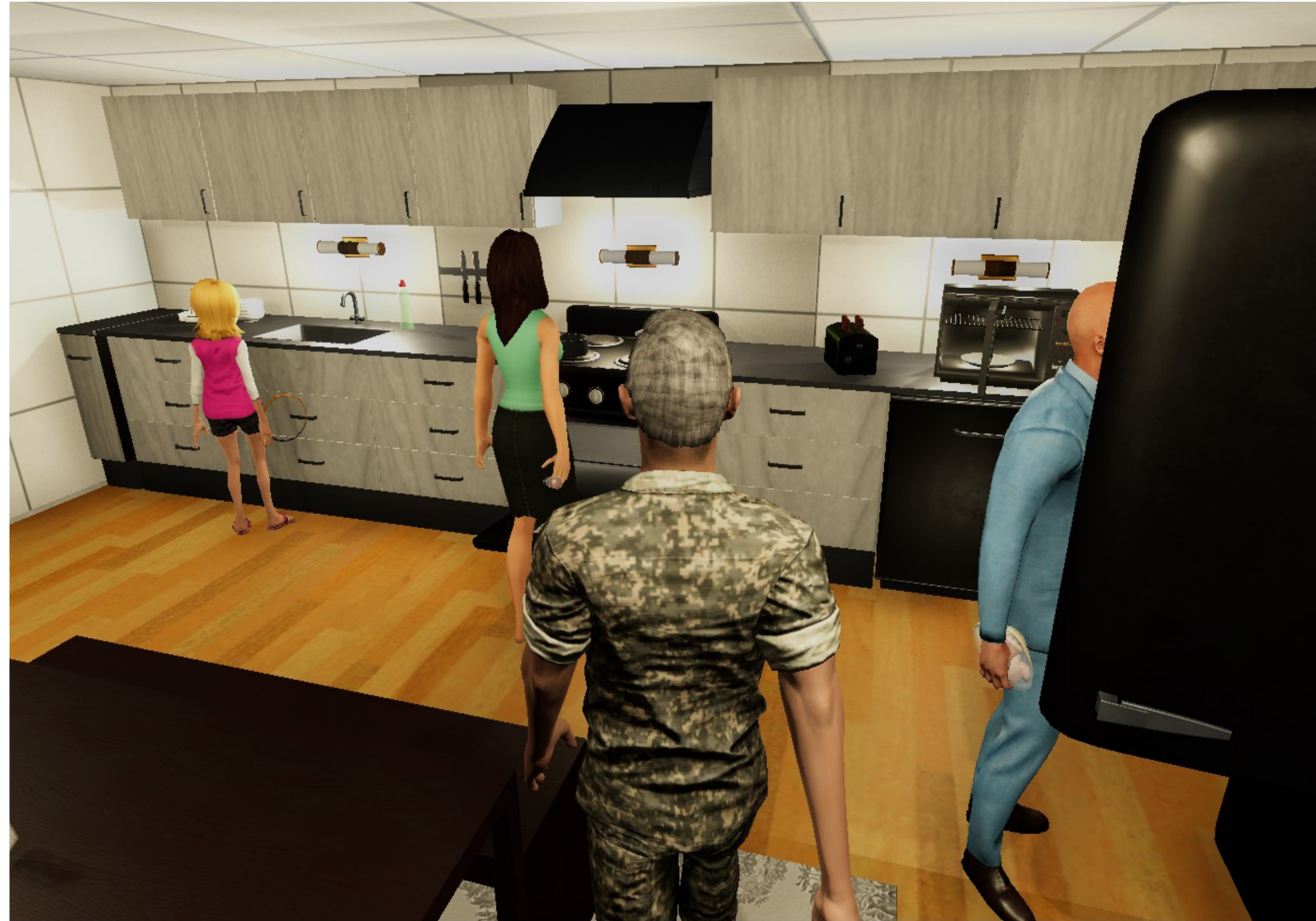}
        \caption{Home Service}
        \label{fig:homeservice}
    \end{subfigure}
    \caption{Simulation Scenario Overview}
    \label{fig:scenario}
\end{figure}

\begin{table*}[h]
	\centering
	\caption{Action predicates for different scenarios}
	\label{table:action_predicates}
	\begin{tabular}{>{\bfseries}c c}
		\toprule
		Scenarios & \textbf{Included Action Predicates} \\
		\midrule
		\multirow{2}{*}{\textbf{Warehouse Mangement}} & \constant{GoToInRoom}, \constant{GoBtwRoom}, \constant{PickUp},\\
		 & \constant{PutInRoom}, \constant{PutNearInRoom}, \constant{Toggle} \\
  
		\midrule
  
		\multirow{3}{*}{\textbf{Home Service}} & \constant{Walk},  \constant{LeftPut}, \constant{LeftPutIn}, \\
  
             & \constant{LeftGrab}, \constant{RightGrab} , \constant{Open}, \constant{Close} \\
             
		& \constant{RightPut}, \constant{RightPutIn},  \constant{SwitchOn},\constant{SwitchOff} \\
		
		\bottomrule
	\end{tabular}
\end{table*}

\subsection{Goal Distribution in Datasets}
In the \textit{Effectiveness of Task-Specific Subtree Pre-Planning} experiments, the datasets involved goals that were composed of 1 to 5 conditions. The overall distribution of these conditions across the datasets is presented in Table \ref{tab:condition_distribution}.

\begin{table}[h]
\centering
\small
\caption{Overall distribution of conditions}
\label{tab:condition_distribution}
\begin{tabular}{lc}
\toprule
\textbf{Condition} & \textbf{Overall Percentage} \\
\midrule
IsOpen/Close(obj) & 11.9\% \\
IsSwitchOn/Off(obj) & 19.5\% \\
In(obj,container) & 27.0\% \\
On(obj,surface) & 41.6\% \\
\bottomrule
\end{tabular}
\end{table}

\subsubsection{Action Allocation Process Based on Homogeneity 
($\boldsymbol{\alpha}$) Parameter}
Let $\alpha$ represent the Homogeneity parameter, where $0 \leq \alpha \leq 1$. Assume there are $n$ robots, and let $\mathcal{A}$ denote the set of actions, with each action $a \in \mathcal{A}$. The cardinality of $\mathcal{A}$ is denoted by $|\mathcal{A}| = m$.

The action allocation process is described as follows:

\begin{enumerate}
    \item Initially, each action $a_i \in \mathcal{A}$ is randomly assigned to a unique robot $r_j \in \{r_1, r_2, \dots, r_n\}$, ensuring that each action is allocated to exactly one robot.
    \item The number of additional robots to which each action will be redundantly assigned is determined by $k = \text{int}(m \times \alpha)$.
    \item Each action $a_i$ is then randomly reallocated to $k$ additional robots, ensuring that these robots do not include the one initially assigned to the action.
\end{enumerate}

\subsection{LLM}

\subsubsection{Prompt}


The objective of this study is to employ the GPT-4 model to generate goal-oriented actions for each robot, utilizing detailed task information to support the pre-planning of subtrees. The input prompt is systematically structured in JSON format, encompassing condition predicates, action predicates, and comprehensive task details, including goals, initial states, objects, and the actions available to each robot. The model is programmed to produce a list of dictionaries, also in JSON format, where each dictionary corresponds to a specific robot and contains multiple sets of goal-related actions derived from the model’s inference. This study takes advantage of GPT-4's capability to handle JSON-structured inputs and outputs, allowing for the definition of a precise JSON schema that ensures strict adherence to the predefined format. This method enhances consistency and predictability, while also streamlining subsequent data processing. For a detailed example of the input prompt, refer to \ref{my_prompt}.

\tcbset{colframe=black, colback=white, arc=2mm, boxrule=0.3mm}

\begin{tcolorbox}
[breakable,title=Prompt,label=my_prompt]
\begin{flushleft}
\textbf{[Condition]}\\
IsNear\_self\_<ALL>, IsOn\_<GRABBABLE>\_<SURFACES>, IsIn\_<GRABBABLE>\_<CONTAINERS>,
IsOpen\_<CAN\_OPEN>, IsClose\_<CAN\_OPEN>,
IsSwitchedOn\_<HAS\_SWITCH>, IsSwitchedOff\_<HAS\_SWITCH>
\vspace{0.5cm}

\textbf{[Action]}\\
Walk\_<ALL>, RightGrab\_<GRABBABLE>, LeftGrab\_<GRABBABLE>,
RightPut\_<GRABBABLE>\_<SURFACES>, LeftPut\_<GRABBABLE>\_<SURFACES>,
RightPutIn\_<GRABBABLE>\_<CONTAINERS>, LeftPutIn\_<GRABBABLE>\_<CONTAINERS>,
RightGrabFrom\_<GRABBABLE>\_<CONTAINERS>, LeftGrabFrom\_<GRABBABLE>\_<CONTAINERS>,
Open\_<CAN\_OPEN>, Close\_<CAN\_OPEN>, SwitchOn\_<HAS\_SWITCH>, SwitchOff\_<HAS\_SWITCH>
\vspace{0.5cm}

\textbf{[Example]}\\
\textbf{[Task Information]}\\
\textbf{"goal"}: ["IsOn(mug,nightstand)", "IsSwitchedOn(tablelamp)", "IsOpen(book)"],\\
\textbf{"init\_state"}: ["IsSwitchedOff(tablelamp)", "IsClose(book)", "IsClose(nightstand)"],\\
\textbf{"objects"}: ["mug", "nightstand", "tablelamp", "book", "remotecontrol", "drawer"],\\
\textbf{"action\_space"}: [\\
  {["Walk", "SwitchOn", "Open", "Close"]},\\ 
  {["Walk", "RightGrab", "RightPut", "SwitchOn", "RightPutIn"]},\\
  {["Walk", "RightGrab", "RightPut", "SwitchOn", "Open", "Close", "RightPutIn"]}
]

\textbf{[Output]}\\
\textbf{"multi\_robot\_subtree\_ls"}: [\{\\
    "WalkToSwitchOntablelamp": ["Walk(self,tablelamp)", "SwitchOn(self,tablelamp)"],\\
    "WalkToOpenBook": ["Walk(self,book)", "Open(self,book)"],\\
    "WalkToOpenNightstand": ["Walk(self,nightstand)", "Open(self,nightstand)"]\\
  \},\\
  \{\\
    "WalkToPutMugOnNightstand": [ "Walk(self,mug)", "RightGrab(self,mug)", "Walk(self,nightstand)", "RightPut(self,mug,nightstand)" ],\\
    "WalkToSwitchOntablelamp": ["Walk(self,tablelamp)", "SwitchOn(self,tablelamp)"],\\
    "WalkToPutRemoteInNightstand": ["Walk(self,remotecontrol)", "RightGrab(self,remotecontrol)", "Walk(self,nightstand)", "RightPutIn(self,remotecontrol,nightstand)"]\\
  \},\\
  \{\\
    "WalkToPutMugOnNightstand": [ "Walk(self,mug)", "RightGrab(self,mug)", "Walk(self,nightstand)", "RightPut(self,mug,nightstand)" ],\\
    "WalkToSwitchOntablelamp": [ "Walk(self,tablelamp)", "SwitchOn(self,tablelamp)"],\\
    "WalkToOpenBook": ["Walk(self,book)", "Open(self,book)" ],\\
    "WalkToPutRemoteInNightstand": ["Walk(self,remotecontrol)", "RightGrab(self,remotecontrol)", "Walk(self,nightstand)", "RightPutIn(self,remotecontrol,nightstand)"],\\
    "WalkToOpenNightstand": [ "Walk(self,nightstand)", "Open(self,nightstand)"] \}]
\vspace{0.5cm}

\textbf{[System]}\\
	\begin{enumerate}[leftmargin=*,topsep=0pt, partopsep=0pt, parsep=0pt, itemsep=1pt]
	 \item For each task, generate all possible composite actions for each robot based on its goals, initial state, and available action space. Repetition of composite actions is permissible. 

	 \item \textbf{[multi\_robot\_subtree\_ls]} is a list where each entry is a dictionary [subtree\_dict] containing all task-related composite actions for a robot. Keys in [subtree\_dict] are composite action names, and values are sequences of atomic actions, ordered such that each action’s effect serves as the precondition for the next. Using the current [Task Information] and [Example], construct \textbf{[multi\_robot\_subtree\_ls]} for each robot.
	 \item The length of \textbf{[multi\_subtree\_list]} equals the number of robots and corresponds to the number of action lists in [action\_space]. With \textbf{\{num\_agent\}} robots, \textbf{[multi\_subtree\_list]} contains \textbf{\{num\_agent\}} dictionaries, each with 1-5 key-value pairs.
	\end{enumerate}

\end{flushleft}
\end{tcolorbox}

\subsubsection{Automatic Reflective Feedback}

 To validate the outputs generated by the LLM, we developed an automated checker that provides feedback in cases of grammatical errors, action sequences that cannot be pre-planned into a subtree, or an insufficient number of action sequences generated for each robot. \\
A specific example is provided in \ref{box:feedback}, where two robots and a key are in Room 0, with the key able to unlock the door to Room 1. The goal is to move a ball from Room 0 to Room 1. The LLM correctly inferred two action sequences: \constant{UseKeyToOpenDoor} and \constant{MoveBallBetweenRooms}. However, the output contained grammatical errors, including incorrect objects, action predicates, and missing action parameters.
Additionally, the \constant{action\_space} list indicated two robots, but the LLM only provided an action sequence dictionary for one robot. Our automated checker detected these issues, generating natural language feedback for the LLM to refine its response. Upon revision, the LLM corrected the grammatical errors and produced two dictionaries representing the \constant{action\_space} for both robots. Subsequently, the pre-planning algorithm attempted to generate subtrees. However, it failed when expanding the \constant{UseKeyToOpenDoor} subtree, as \constant{Toggle(self,door-0)} could not be extended after \constant{PickUp(self,ball-0)}. The failure occurred because the add effects of \constant{Toggle(self,door-0)} (\constant{IsOpen(door-0)}) do not satisfy the preconditions of \constant{PickUp(self,ball-0)} (\constant{IsHandEmpty(self), IsIn(ball-0,room-0), IsIn(self,room-0)}). The automated checker identified this mismatch and applied the reflective feedback template, which was then provided to the LLM for further refinement. Ultimately, the LLM successfully inferred a valid sequence.

\begin{tcolorbox}
[breakable,title=An Example of LLM Process with Automatic Reflective Feedback,label=my_feedback,breakable]
\begin{flushleft}
\textbf{[Task Information]}\\
\textbf{"goal"}: ["IsInRoom(ball-0,room-1)"],\\
\textbf{"init\_state"}: ["IsInRoom(ball-0,room-0)", "IsInRoom(key-0,room-0)","IsClose(door-0)"],"IsInRoom(robot-0,room-0)", "IsInRoom(robot-1,room-0)","IsInRoom(robot-1,room-0)","IsHandEmpty(robot-0)","IsHandEmpty(robot-1)"]\\
\textbf{"objects"}: ["key-0", "ball-0", "door-0", "room-0", "room-1"],\\
\textbf{"action\_space"}: [\\
  {["GoToInRoom", "PickUp", "PutInRoom", "Toggle"]},\\ 
  {["GoToInRoom", "GoBtwRoom", "PickUp", "PutInRoom"]}]\\
\vspace{0.5cm}
\textbf{[LLM]}\\
\textbf{"multi\_robot\_subtree\_ls"}:[\\
    "UseKeyToOpenDoor": ["GoToInRoom(self,key-0,room-0)", "PickUp\underline{(key-0)}", "GoToInRoom(self,door-0,room-0)", "\underline{Open}(self,door-0)"],\\

     "MoveBallBetweenRooms": ["GoToInRoom(self,ball-0,room-0)", "PickUp(self,\underline{ball-1})", "GoBtwRoom(self,room-0,room-1)","PutInRoom(self,ball-0,room-1)"]\\

  \}]

\vspace{0.2cm}
\tikz{\draw[dashed,line width=1pt] (0,0) -- (\linewidth,0);}\\ 
\vspace{0.2cm}
\textbf{Reflective Feedback:}\newline
1. The object \textbf{["ball-1"]} does not exist and must be an element of the set \textbf{[objects]}.\\
2. The action predicates \textbf{["PickUp(key-0)","Open"]} either do not exist or are incorrectly formatted. They must be part of the set \textbf{[action\_space]}.\\
3. The number of robots involved in this task is \textbf{\{num\_agent\}}, which implies that the \textbf{[multi\_subtree\_list]} should contain \textbf{\{num\_agent\}} dictionaries. Each of these dictionaries should have 1-4 key-value pairs. Please revise accordingly.\\

\vspace{0.5cm}

\textbf{[LLM]}\\
\textbf{"multi\_robot\_subtree\_ls"}: [\{\\
    "UseKeyToOpenDoor": ["GoToInRoom(self,key-0,room-0)", "PickUp(self,key-0)", "GoToInRoom(self,door-0,room-0)", "Toggle(self,door-0)", "\underline{PickUp(self,ball-0)}"]\\
  \},\\
  \{\\
    "MoveBallBetweenRooms": ["GoToInRoom(self,ball-0,room-0)", "PickUp(self,ball-0)", "GoBtwRoom(self,room-0,room-1)", "PutInRoom(self,ball-0,room-1)"]\\
  \}]

\vspace{0.2cm}
\tikz{\draw[dashed,line width=1pt] (0,0) -- (\linewidth,0);}\\ 
\vspace{0.2cm}
\textbf{Reflective Feedback:}\newline

1.\textbf{"UseKeyToOpenDoor"} cannot be pre-planned. \\ 
2.\textbf{"Toggle(self,door-0)"} and \textbf{"PickUp(self,ball-0)"} cannot be pre-planned together.\\
3.[Toggle(self,door-0).add] = ("IsOpen(door-0)")\\
4.[PickUp(self,ball-0).pre] = ("IsHandEmpty(self)", "IsIn(ball-0,room-0)", "IsIn(self,room-0)")\\
5.[Toggle(self,door-0).add] condition does not belong to the [PickUp(self,ball-0).pre] conditions.

\vspace{0.5cm}

\textbf{[LLM]}\\
\textbf{"multi\_robot\_subtree\_ls"}: [\{\\
    "UseKeyToOpenDoor": ["GoToInRoom(self,key-0,room-0)", "PickUp(self,key-0)", "GoToInRoom(self,door-0,room-0)", "Toggle(self,door-0)"]\\
  \},\\
  \{\\
    "MoveBallBetweenRooms": ["GoToInRoom(self,ball-0,room-0)", "PickUp(self,ball-0)", "GoBtwRoom(self,room-0,room-1)", "PutInRoom(self,ball-0,room-1)"]\\
  \}]

\end{flushleft}
\label{box:feedback}
\end{tcolorbox}

\end{document}